\documentclass{article}
\usepackage[utf8]{inputenc}
\usepackage{amsmath}
\usepackage{amssymb}
\usepackage{amsthm}
\usepackage[OT1]{fontenc}
\usepackage{subcaption} 
\usepackage{caption}

\usepackage{graphicx}
\usepackage{fullpage}
\usepackage{bm}
\usepackage{algorithm}
\usepackage{algorithmic}
\usepackage[colorlinks=true, linkcolor=blue, citecolor=blue]{hyperref}
\usepackage{cleveref}
\usepackage{natbib}
\usepackage{def}

\bibpunct{(}{)}{;}{a}{,}{,}

\def\P{{\mathbb P}}
\def\E{{\mathbb E}}

\def\R{{\mathbb R}}

\def\dimpop{{d_{\text {eff }}^{\text{pop}}}}
\def\dimsamp{{d_{\text {eff }}^{\text{sample}}}}

\def\Ical{{\mathcal I}}
\def\Xcal{{\mathcal{X}}}
\def\Ncal{{\mathcal{N}}}
\def\Fcal{{\mathcal{F}}}

\def\Dcal{{\mathcal D}}
\def\Scal{{\mathcal S}}

\def\Acal{{\mathcal A}}
\def\Hcal{{\mathcal H}}
\def\Xcal{{\mathcal X}}

\def\Mcal{{\mathcal M}}
\def\Qcal{{\mathcal Q}}

\def\thisstate{{s_h}}
\def\thisaction{{a_h}}
\def\thissai{{s_h^i, a_h^i}}
\def\thisreward{{r_h}}

\def\nextstatei{{s_{h+1}^i}}

\def\thisq{{Q_h}}

\def\thissa{{s_h, a_h}}

\def\thissai{{s_h^i, a_h^i}}

\def\expectV{{\P_h V_{h+1}}}
\def\behpolicy{{\pi_{b,h}}}
\def\learnedpolicy {{\widehat{\pi}_h}}

\def\learnedq {{\widehat{Q}_h}}
\def\learnedv {{\widehat{V}_{h}}}
\def\learnednextv {{\widehat{V}_{h+1}}}
\def\bracketing {{N_{[]}}}
\def\learnedr {{\widehat{r}_h(s_h,a_h)}}
\def\sumn{{\sum_{i=1}^n}}

\def\behaviorq{{Q_h^{\pi_b,\gamma}}}
\def\behaviorv{{V_h^{\pi_b,\gamma}}}

\usepackage{amsmath}
\usepackage{amssymb}
\usepackage{mathtools}
\usepackage{amsthm}

\newtheorem{assumption}[theorem]{Assumption}

\title{Reinforcement Learning with Human Feedback: Learning Dynamic Choices via Pessimism}
\author{\normalsize Zihao Li\thanks{Prineton University;
			\texttt{zihaoli@princeton.edu}}\qquad Zhuoran Yang\thanks{Yale University; \texttt{zhuoran.yang@yale.edu}}\qquad Mengdi Wang\thanks{Princeton University; \texttt{mengdiw@princeton.edu}}}
\date{May 2023}

\begin{document}

\maketitle
\begin{abstract}
In this paper, we study offline Reinforcement Learning with Human Feedback (RLHF) where we aim to learn the human's underlying reward and the MDP's optimal policy from a set of trajectories induced by human choices.   RLHF is challenging for multiple reasons: large state space but limited human feedback, the bounded rationality of human decisions, and the off-policy distribution shift. In this paper, we focus on the 
Dynamic Discrete Choice (DDC) model for modeling and understanding human choices. DCC, rooted in econometrics and decision theory, is widely used to model a human decision-making process with forward-looking and bounded rationality.  
We propose a \underline{D}ynamic-\underline{C}hoice-\underline{P}essimistic-\underline{P}olicy-\underline{O}ptimization (DCPPO) method. \ The method involves a three-stage process: The first step is to estimate the human behavior policy and the state-action value function via maximum likelihood estimation (MLE); the second step recovers the human reward function via  minimizing Bellman mean squared error using the learned value functions; the third step is to plug in the learned reward and invoke pessimistic value iteration for finding a near-optimal policy. With only single-policy coverage (i.e., optimal policy) of the dataset, we prove that the  suboptimality of DCPPO \textit{almost} matches the classical pessimistic offline RL algorithm in terms of suboptimality’s dependency on distribution shift and dimension. To the best of our knowledge, this paper presents the first theoretical guarantees for off-policy offline RLHF with dynamic discrete choice model.
\end{abstract}

\section{Introduction}
\textit{Reinforcement Learning with Human Feedback} (RLHF) is an area  in machine learning research that incorporates human guidance or feedback to learn an optimal policy. In recent years, RLHF has achieved significant success in clinical trials, auto-driving, robotics, and large language models \citep[e.g.][]{ouyang2022training, gao2022scaling,glaese2022improving,hussein2017imitation,jain2013learning,kupcsik2018learning,menick2022teaching,nakano2021webgpt,novoseller2020dueling}. Unlike conventional offline reinforcement learning, where the learner aims to determine the optimal policy using observable reward data, in RLHF, the learner does not have direct access to the reward signal but instead can only observe a historical record of visited states and human-preferred actions. In such cases, the acquisition of reward knowledge becomes pivotal.

  \textit{Dynamic Discrete Choice} (DDC) model is a framework for studying learning for human choices from data, which has been extensively studied in econometrics literature \citep{rust1987optimal,hotz1993conditional, hotz1994simulation,aguirregabiria2002swapping,kalouptsidi2021linear,bajari2015identification,chernozhukov2022locally}. In a DDC model, the agent make decisions under unobservable perturbation, i.e. $\pi_h(a_h\mid s_h) = \operatorname{argmax}_a\{Q_h(s_h,a)+\epsilon_h(a)\}$, where $\epsilon_h$ is an unobservable random noise and $Q_h$ is the agent's action value function. 
 
 In this work, we focus on RLHF within the context of a dynamic discrete choice model. Our challenges are three-folded: 
     (i) The agent must first learn the human behavior policies from the feedback data.
      (ii)  As the agent's objective is to maximize cumulative reward, the reward itself is not directly observable. We need to estimate the reward from the behavior policies. 
      (iii) We face the challenge of insufficient dataset coverage and large state space.
      
 With these coupled challenges, we ask the following question:
\begin{center}
	\textit{Without access to the reward function, can one learn the optimal pessimistic policy from merely human choices under the dynamic choice model? }
\end{center}

\paragraph{Our Results.}
In this work, we propose the \underline{D}ynamic-\underline{C}hoice-\underline{P}essimistic-\underline{P}olicy-\underline{O}ptimization (DCPPO) algorithm. 
By addressing challenges (i)-(iii), our contributions are three folds:
    (i) For learning behavior policies in large state spaces, we employ maximum likelihood estimation to estimate state/action value functions with function approximation. We establish estimation error bounds for general model class with low covering number.
    (ii) Leveraging the learned value functions, we minimize the Bellman mean squared error (BMSE) through linear regression. This allows us to recover the unobservable reward from the learned policy. Additionally, we demonstrate that the error of our estimated reward can be efficiently controlled by an uncertainty quantifier. 
    (iii) To tackle the challenge of insufficient coverage, we follow \textit{the principle of pessimism}, by incorporating a penalty into the value function during value iteration. We establish the suboptimality of our algorithm with high probability with only single-policy coverage. 
    
    Our result matches existing pessimistic offline RL algorithms in terms of suboptimality’s dependence on distribution shift and dimension,  even in the absence of an observable reward. To the best of our knowledge, our results offer the first theoretical guarantee for pessimistic RL under the human dynamic choice model.

 \subsection{Related Work}
\paragraph{Reinforcement Learning with Human Feedback.} In recent years RLHF and inverse reinforcement learning (IRL) has been widely applied to robotics, recommendation system, and large language model \citep{ouyang2022training, lindner2022interactively,menick2022teaching,jaques2020human,lee2021pebble,nakano2021webgpt}. However, there are various ways to incorporate human preferences or expertise into the
decision-making process of an agent.  \citep{shah2015estimation,ouyang2022training, saha2022efficient} learn reward from pairwise comparison and ranking. \citep{chen2022humanintheloop,kong2023provably} study Human-in-the-Loop RL. \cite{pacchiano2021dueling} study pairwise comparison with function approximation in pairwise comparison. \cite{zhu2023principled} study various cases of preference-based-comparison in contextual bandit problem with linear function approximation, however convergence of their algorithm relies on the implicit assumption of sufficient coverage. The majority of prior research in RLHF only consider bandit cases and have not studied  MDP case with transition dynamics. \cite{wang2018exponentially} study how to learn a uniformly better policy of an MDP from an offline  dataset by learning the advantage function. However, they cannot guarantee the learned policy converges to the optimal policy.\\
\paragraph{Dynamic Discrete Choice Model.} Dynamic Discrete Choice (DDC) model is a widely studied choice model in econometrics and is closely related to reward learning in IRL and RLHF. In the DDC model, the human agent is assumed to make decisions under the presence of Gumbel noise (Type I Extreme Error)\citep{aguirregabiria2002swapping,chernozhukov2022locally,bajari2015identification,kalouptsidi2021linear,adusumilli2019temporal}, i.e. under bounded rationality, and the task is to infer the underlying utility. A method highly related to our work is the \textit{conditional choice probability} (CCP) algorithm \citep{hotz1993conditional,arcidiacono2011practical,bajari2015identification,adusumilli2019temporal}, in which the learner first estimate choice probability from the dataset, and then recover the underlying value function from the estimated dynamic choices. However, most work in econometrics cares for asymptotic $\sqrt{n}$-convergence of estimated utility and does not study finite sample estimation error. Moreover, their methods suffer from significant computation burdens from large or high dimensional state space \citep{zeng2022structural}. In recent years, there has been work combining the dynamic discrete choice model and IRL. \cite{zeng2022structural} prove the equivalence between DDC estimation problem and
maximum likelihood IRL problem, and propose an online gradient method for reward estimation under ergodic dynamics assumption. \cite{zeng2023understanding} reformulate the reward estimation in the DDC model
into a bilevel optimization and propose a model-based approach by assuming an environment simulator.
\\
\paragraph{Offline Reinforcement Learning and Pessimism.} The idea of introducing pessimism for offline RL to deal with distribution shift has been studied in recent
years \citep{jin2021pessimism,bai2022pessimistic,zhong2022pessimistic,uehara2021representation,yu2020mopo,rashidinejad2021bridging,shi2022pessimistic}. Importantly, \cite{jin2021pessimism} show that pessimism is sufficient to eliminate spurious correlation and intrinsic uncertainty when doing value iteration. \cite{uehara2021representation} show that with single-policy coverage, i.e. coverage over the optimal policy, pessimism is sufficient to guarantee a $\Ocal(n^{-1/2})$ suboptimality. In this paper, we connect RLHF with offline RL and
show our algorithm achieves pessimism by designing an uncertainty quantifier that can tackle the error inherited from estimating reward functions, which is crucial in pessimistic value iteration.

\subsection{Notations and Preliminaries}
For a positive-semidefinite matrix $A\in \R^{d\times d}$ and vector $x\in\R^d$, we use $\|x\|_A$ to denote $\sqrt{x^\top Ax}$.  For an arbitrary space $\Xcal$, we use $\Delta(\Xcal)$ to denote the set of all probability distribution over $\Xcal$. For two vectors $x,y\in \R^d$, we denote $x\cdot y = \sum_i^d x_iy_i$ as the inner product of $x,y$. We denote the set of all probability measures on $\Xcal$ as $\Delta(\Xcal)$. We use $[n]$ to represent the set of integers from $0$ to $n-1$.  For every set $\Mcal\subset\Xcal$ for metric space $\Xcal$, we define its $\epsilon$-covering number with respect to norm  $\|\cdot\|$ by $N(\Mcal, \|\cdot\|,\epsilon)$. We define a finite-horizon MDP model $M = (\Scal,\Acal, H, \{P_h\}_{h\in[H]},\{r_h\}_{h\in[H]} )$, $H$ is the horizon length,  in each step $h\in[H]$ , the agent starts from state $\thisstate$ in the state space $\Scal$, chooses an action  $\thisaction \in \Acal$ with probability $\pi_{h}(a_h\mid s_h)$ , receives a reward of $\thisreward(\thissa)$ and transits to the next state $s'$ with probability $P_h(s'\mid s_h,a_h)$. Here $\Acal$ is a finite action set with $|\Acal|$ actions and $P_h(\cdot| \thissa) \in \Delta(\thissa)$ is the transition kernel condition on state action pair $(s, a)$. For convenience we assume that $r_h(s,a) \in [0,1]$ for all $(s,a, h) \in \Scal \times \Acal\times[H]$. Without loss of generality, we assume that the initial state of each episode $s_0$ is fixed. Note that this will not add difficulty to our analysis. For any policy $\pi = \{\pi_h\}_{h\in[H]}$ the state value function is $
V_h^\pi(s) = \E_\pi\big[\sum_{t = h}^H r_t(s_t, a_t)\bigm\vert s_h = s\big],
$
and the action value function is
$
Q_h^\pi(s,a) = \E_\pi\big[\sum_{t = h}^H r_t(s_t, a_t)\bigm\vert s_h = s,a\big],
$
here the expectation $\E_\pi$  is taken
with respect to the randomness of the trajectory induced by
$\pi$, i.e. is obtained by taking action $a_t \sim \pi_t(\cdot\mid s_t)$ and observing $s_{t+1} \sim P_h(\cdot\mid s_t,a_t)$. For any function $f: \Scal\rightarrow \R$, we define the transition operator $
\P_h f(s,a) = \E[f(s_{h+1})\mid s_h =s,a_h =a].
$
We also define the Bellman equation for any policy $\pi$, $
        V_h^{\pi}(s) = \langle \pi_{h}(a\mid s), Q_h^{\pi_b}(s,a) \rangle, 
    Q_h^{\pi}(s,a) = r_h(s,a) +  \P_h V^{\pi}_{h+1}(s,a).
$ For an MDP we denote its optimal policy as $\pi^*$, and define the performance metric for any policy $\pi$ as $
\operatorname{SubOpt}(\pi)  = V_1^{\pi^*} - V_1^{\pi}.
$

\section{Problem Formulation}
In this paper, we aim to learn from a dataset of human choices under dynamic discrete choice model. Suppose we are provided with dataset $\Dcal = \{\Dcal_h =  \{\thissai\}_{i\in [n]}\}_{h\in[H]}$, containing $n$ trajectories collected by observing a single human behavior in a dynamic discrete choice model.
 Our goal is to learn the optimal policy $\pi^*$ of the underlying MDP. We assume that the agent is bounded-rational and makes decisions according to the dynamic discrete choice model \citep{rust1987optimal,hotz1993conditional,chernozhukov2022locally,zeng2023understanding}. In dynamic discrete choice model, the agent's policy has the following characterization \citep{rust1987optimal,aguirregabiria2002swapping,chernozhukov2022locally},
\begin{equation}\label{eq:behavior-policy}
    \pi_{b,h}(a\mid s) = \frac{\exp(\behaviorq(s,a))}{\sum_{a'\in\Acal} \exp(\behaviorq(s,a'))},
\end{equation}
here $\behaviorq(s,a)$ works as the solution of the discounted Bellman equation, \begin{align}\label{eq:discounted-bellman}
    \behaviorv(s) = \langle \pi_{b,h}(a\mid s), \behaviorq(s,a) \rangle,  \qquad
    \behaviorq(s,a) = r_h(s,a) + \gamma\cdot \P_h V^{\pi_b,\gamma}_{h+1}(s,a)
\end{align} for all $(s,a)\in\Scal\times\Acal$. Note that \eqref{eq:discounted-bellman} differs from the original Bellman equation due to the presence of $\gamma$, which is a  discount factor in $[0,1]$, and measures the myopia of the agent. The case of $\gamma = 0$ corresponds to a \textit{myopic} human agent. Such choice model comes from  the perturbation of noises,$$
\pi_{b,h}(\cdot\mid s_h) = \operatorname{argmax}_{a\in\Acal}\bigg\{ r_h(s_h,a) + \epsilon_h(a) + \gamma\cdot\P_h V^{\pi_b,\gamma}_{h+1}(s_h,a)\bigg\},
$$
where $\{\epsilon_h(a)\}_{a\in\Acal}$ are i.i.d Gumbel noises that are observed by the agent but not the learner, $\{V_{h}^{\gamma,\pi_b}\}_{h\in[H]}$ is the value function of the agent. Such model is widely used to model human decision \citep{rust1987optimal,aguirregabiria2002swapping,chernozhukov2022locally,aguirregabiria2010dynamic,arcidiacono2011practical,bajari2015identification}. We also remark that the state value function defined in \eqref{eq:discounted-bellman} corresponds to the \textit{ex-ante} value function in econometric studies \citep{aguirregabiria2010dynamic,arcidiacono2011practical,bajari2015identification}.  When considering Gumbel noise as part of the reward, the value function may have a different form. However, such a difference does not add complexity to our analysis.

\section{Reward Learning from Human Dynamic Choices}
In this section, we present a general framework of an offline algorithm for learning the reward of the underlying MDP.   Our algorithm consists of two steps:
    (i) The first step is to estimate the agent behavior policy from the pre-collected dataset $\Dcal$ by maximum likelihood estimation (MLE). Motivated by Holtz-Miller inversion in Conditional Choice Probability (CCP) method \citep{hotz1994simulation,adusumilli2019temporal}, we can recover the action value functions $\{\behaviorq\}_{h\in[H]}$ from  \eqref{eq:behavior-policy} and the state value functions $\{\behaviorv\}_{h\in [H]}$ from \eqref{eq:discounted-bellman} using function approximation. In Section \ref{sec:first-step-general}, we  analyze the error of our estimation and prove that for any model class  with a small covering number, the error from MLE estimation is of scale $\tilde{\Ocal}(1/n)$ in dataset distribution. We also remark that our result does not need the dataset to be well-explored, which is implicitly assumed in previous works \citep{zhu2023principled,chen2020dynamic}.
    (ii) We recover the underlying reward from the model class  by minimizing a penalized Bellman MSE with plugged-in value functions learned in step (i). In Section \ref{sec:linear MDP}, we study linear model MDP as a concrete example.   Theorem \ref{thm:reward-est} shows that the error of estimated reward can be bounded by an elliptical potential term for all $(s, a)\in \Scal\times\Acal$ in both settings.  
First, we make the following assumption for function approximation.
\begin{assumption}[\textbf{Function Approximation Model Class}]\label{ass:funct-approx}
We assume the existence of a model class  $\Mcal = \{\Mcal_h\}_{h\in[H]}$ containing functions $f:\Scal\times\Acal\rightarrow[0,H]$  for every $h\in[H]$, and is rich enough to capture $r_h$ and $Q_h$, i.e. $r_h\in\Mcal_h$, $Q_h \in \Mcal_h$. We also assume a positive penalty $\rho(\cdot)$ defined on $\Mcal$.
\end{assumption}
Assumption \ref{ass:funct-approx} requires that the model class $\Mcal_h$ is rich enough such that it contains the true model. In practice, we can choose $\Mcal_h$ to be a class of neural networks or kernel functions.
\begin{algorithm}[H]
   \caption{DCPPO: Reward Learning for General Model Class}
   \begin{algorithmic}[1]\label{alg:reward-learn}
   \REQUIRE  Dataset $ \big\{\Dcal_h =  \{\thissai\}_{i\in [n]}\big\}_{h\in[H]}$, constant $\lambda>0$, penalty function $\rho(\cdot)$, parameter $\beta$.
   	 \FOR{step $h=H, \dots,1$}
   	 \STATE Set $\widehat{Q}_h = \operatorname{argmax}_{Q\in \Mcal_h}\frac{1}{n} \sum_{i=1}^n Q(\thissai)- \log\big(\sum_{a'\in\Acal} \exp(Q(s_h^i,a'))\big)$.

     \STATE Set $    \learnedpolicy(a_h\mid s_h) = \exp(\learnedq(\thissa))/\sum_{a'\in\Acal} \exp(\learnedq(s_h,a')$.
   \STATE Set $\widehat{V}_h(s_h) = \langle \widehat{Q}_h(s_h,\cdot), \widehat{\pi}_h(\cdot\mid s_h)\rangle_\Acal$.
   \STATE Set $\learnedr = \operatorname{argmin}_{r\in \Mcal_h} \big\{ \sum_{i=1}^n \big(\thisreward(\thissai) + \gamma\cdot \learnednextv(\nextstatei) - \learnedq(\thissai)\big)^2 + \lambda \rho(r) \big\}$.\label{line:bell-mse}
   	 \ENDFOR\\
    \STATE \textbf{Output:} $\{\widehat{r}_h\}_{h\in[H]}$.
   \end{algorithmic}
\end{algorithm}
We highlight that for value function estimation, Algorithm \ref{alg:reward-learn} matches existing econometric literature in CCP method for dynamic choice model \citep{hotz1993conditional,hotz1994simulation,bajari2015identification,adusumilli2019temporal}: we first estimate the choice probability of the agent, and then recover the value functions. Moreover, for examples such as linear MDP or RKHS, our algorithm can provably learn the optimal policy with only single-policy coverage (instead of sufficient coverage in previous works, e.g. see Assumption 3 in \cite{hotz1993conditional}, Assumption 2 in \cite{adusumilli2019temporal}, and the references therein) for a finite-sample dataset. Our algorithm also needs no additional computational oracles.
\subsection{First Step: Recovering Human Policy and Human State-Action Values}\label{sec:first-step-general}

For every step $h$, we use maximum likelihood estimation (MLE) to estimate the behavior policy $\behpolicy$, corresponding to $\behaviorq(s,a)$ in a general model class $\Mcal_h$. For each step $h\in[H]$, we have the log-likelihood function \begin{equation}\label{eq:mle}
  {L}_h(Q) =\frac{1}{n} \sum_{i=1}^n\log\bigg(\frac{\exp({Q}(\thissai))}{\sum_{a'\in\Acal} \exp(Q(s,a'))}\bigg)  
\end{equation}
for $Q \in \Mcal_h$,
and we estimate $\thisq$ by maximizing \eqref{eq:mle}. Note that by Equation \eqref{eq:behavior-policy}, adding a constant on $\behaviorq$ will produce the same policy under dynamic discrete model. Thus the real behavior value function is unidentifiable in general. For identification,  we have the following assumption.
\begin{assumption}[\textbf{Model Identification}]\label{ass:identify}
We assume that there exists one $a_0\in\Acal$, such that $Q(s,a_0) = 0$ for every $s\in\Scal$. 
\end{assumption}
 Note that this assumption  does not affect our further analysis. Other identifications include parameter constraint \citep{zhu2023principled} or utility constraints \citep{bajari2015identification}.
We can ensure the estimation of the underlying policy and corresponding value function is accurate in the states the agent has encountered. Formally, we have the following theorem, \begin{theorem}[\textbf{Policy and Value Functions Recovery from Choice Model}]\label{thm:emp-mle-guarantee}
   With Algorithm \ref{alg:reward-learn} , we have $$
\E_{\Dcal_h}\big[\|\learnedpolicy(\cdot\mid\thisstate) - \behpolicy(\cdot\mid \thisstate)\|_1^2\big] \leq \Ocal\bigg( \frac{\log\big(H\cdot N(\Mcal_h, \|\cdot\|_{\infty}, 1/n )/\delta\big)}{n}\bigg)
    $$ and 
    $$
\E_{\Dcal_h}\big[\|\learnedq(\thisstate,\cdot) - \behaviorq(\thisstate, \cdot)\|_1^2\big] \leq \Ocal\bigg( \frac{H^2 e^{2H} \cdot |\Acal|^2\cdot \log\big(H\cdot N(\Mcal_h, \|\cdot\|_{\infty}, 1/n )/\delta\big)}{n}\bigg)
$$
hold for every $h\in[H]$ with probability at least $1-\delta$. Here $\E_{\Dcal_h}[\cdot]$ means the expectation is taken on collected dataset $\Dcal_h$, i.e. the mean value taken with respect to $\{s_h^i\}_{i\in [n]}$. 
\end{theorem}
\begin{proof}
    See Appendix \ref{app:proof-policy-recover} for details.
\end{proof}
Theorem \ref{thm:emp-mle-guarantee} shows that we can efficiently learn $\behpolicy$ from the dataset under identification assumption. Specifically, we have a $\Ocal(1/n)$ guarantee, which only depends on the model size via log covering number. This coincides with finite-sample MLE guarantee \citep{uehara2021pessimistic,ge2023provable,agarwal2020flambe,zhan2022pac}. Here the $\exp(H)$ comes from the exponential dependence of $\pi_h^b$ on $\behaviorq$. As a result, we can provably recover the value functions by definition in  \eqref{eq:behavior-policy}.
\subsection{Reward Learning from Dynamic Choices}\label{sec:linear MDP} 
To estimate reward function $r_h$ from $\widehat{Q}_h$ and $\widehat{V}_{h+1}$ obtained in Algorithm \ref{alg:reward-learn}, we have to regress $\widehat{Q}_h - \widehat{V}_{h+1}$ on $(s,a)$. However, when analyzing such regression, we have to specify the class structure of $\P_h\widehat{V}_{h+1}$, which further depends on $\widehat{\pi}_{b,h+1}\propto \exp(\widehat{Q}_{h+1})$.  For ease of presentation, let us just focus on linear class. Formally, we consider the function class $\Mcal_h = \{f(\cdot) = \phi(\cdot)^\top \theta: \Scal \times \Acal \rightarrow \R, \theta \in \Theta\}$ for $h\in[H]$, where $\phi\in\R^d$ is the feature defined on $\Scal\times \Acal$, $\Theta$ is a subset of $\R^d$ which parameterizes the model class, and $d > 0$ is the dimension of the feature. Corresponding to Assumption \ref{ass:identify},  we also assume that $\phi(s,a_0) = 0$ for every $s\in\Scal$. Note that this model class contains the reward $r_h$ and state action value function $Q_h$ in tabular MDP  where $\phi(s,a)$ is the one-hot vector of $(s,a)$. The linear model class also contains linear MDP, which assumes both the transition $P(s_{h+1}\mid s_h, a_h)$ and the reward $r_h(s_h, a_h)$ are linear functions of feature $\phi(s_h,a_h)$ \citep{jin2020provably, duan2020minimax,jin2021pessimism}. In linear model case, our first step MLE in \eqref{eq:mle} turns into a logistic regression, \begin{equation}\label{eq:log-reg}
    \widehat{\theta}_h=\operatorname{argmax}_{\theta\in \Theta}\frac{1}{n} \sum_{i=1}^n\phi(\thissai)\cdot \theta- \log\bigg(\sum_{a'\in\Acal} \exp(\phi(s_h^i,a')\cdot\theta)\bigg),
\end{equation}
which can be efficiently solved by existing state-of-art optimization methods. 
  We now have 
$       \{\widehat{Q}_h\}_{h\in[H]} ,
    \{\learnedpolicy\} _{h\in[H]}
$
 and $
 \{\learnedv \}_{h\in[H]}
 $ in Algorithm \ref{alg:reward-learn}
to be our estimations for $\{\behaviorq\}_{h\in[H]}, \{\behpolicy\}_{h\in[H]}$ and $\{\behaviorv\}_{h\in[H]}$.
The second stage estimation in Line \ref{line:bell-mse} of Algorithm \ref{alg:reward-learn} now turns into a ridge regression for the Mean Squared Bellman Error, with $\rho(\phi\cdot w)$ being $\|w\|_2^2$, \begin{equation}\label{eq:ridge-reg}
    \widehat{w}_h = \operatorname{argmin}_w  \bigg\{ \sum_{i=1}^n \bigg(\phi(\thissai)\cdot w + \gamma\cdot \learnednextv(\nextstatei) - \learnedq(\thissai)\bigg)^2 + \lambda \|w\|^2_2 \bigg\}. 
\end{equation}
Note that optimization \eqref{eq:ridge-reg} has a closed form solution, \begin{align}\label{eq:ridge-sol}
    \widehat{w}_h = (\Lambda_h + \lambda I)^{-1} \bigg( &\sum_{i=1}^n \phi(\thissai)\big(\learnedq(\thissai) - \gamma\cdot\learnednextv(\nextstatei)\big)\bigg),\\
    \text{where }\Lambda_h = &\sum_{i=1}^n \phi(\thissai)\phi(\thissai)^\top,
\end{align}
and we set $\widehat{r}(\thissa) = \phi(\thissa)\cdot \widehat{w}_h$. We also make the following assumption on the model class $\Theta$ and the feature function.




\begin{assumption}[\textbf{Regular Conditions}]\label{ass:regular}
    We assume that: 
        (i) For all $\theta \in \Theta$, we have $\|\theta\|_2 \leq H\sqrt{d}$; for reward $r_h = \phi\cdot w_h$, we assume $\|w_h\|_2\leq \sqrt{d}$.
        (ii) For all $(\thissa)\in \Scal\times \Acal$, $\|\phi(\thissa)\|_2 \leq 1$.
        (iii) For all $n>0$,  $\log N(\Theta,\|\cdot\|_\infty,1/n) \leq c\cdot d\log n$ for some absolute constant $c$.
\end{assumption}

We are now prepared to present our main result:
\begin{theorem}[\textbf{Reward Estimation for Linear Model MDP}]\label{thm:reward-est}
    With Assumption \ref{ass:funct-approx}, \ref{ass:regular}, the  estimation of our reward function holds with probability  $1-\delta$ for all $(s,a) \in \Scal \times \Acal$ and all $\lambda >0$,
    \begin{align}
        &|r_h(s,a) - \widehat{r}_h(s,a) |\nonumber \\
        &\quad \leq \|\phi(s,a)\|_{(\Lambda_h+\lambda I)^{-1}}\cdot {\Ocal}\bigg(\sqrt{\lambda d}+ (1+\gamma)\cdot{He^{H}}\cdot |\Acal|\cdot d\sqrt{\log\big({nH}/{\lambda\delta}\big)}\bigg).
    \end{align}
\end{theorem}
\begin{proof}
See Appendix \ref{sec:prove-reward-est}  for details.
\end{proof}
Note that the error can be bounded by the product of two terms, the elliptical potential term $\|\phi(s,a)\|_{(\Lambda+\lambda\cdot I)^{-1}}$ and the norm of a self normalizing term of scale $O(He^{H}\cdot|\Acal|\cdot d\sqrt{\log(n/\delta)})$. Here the exponential dependency $\Ocal(e^{H}|\Acal|)$ comes from estimating $\behaviorq$ with logistic regression and also occurs in logistic bandit \citep{zhu2023principled,fei2020risk}. It remains an open question if this additional factor can be improved, and we leave it for future work.
\begin{remark}
We remark that except for the exponential term in $H$, Theorem \ref{thm:reward-est} \textit{almost} matches the result when doing linear regression on an observable reward dataset, in which case error of estimation is of scale $\tilde{\Ocal}(\|\phi(s, a)\|_{(\Lambda+\lambda I)^{-1}}\cdot dH)$ \citep{ding2021provably,jin2021pessimism}. When the human behavior policy has sufficient coverage, i.e. the minimal eigenvalue of $\E_{\pi_b}[\phi\phi^\top]$, $\sigma_{\min}(\E_{\pi_b}[\phi\phi^\top])>c>0$, we have $\|\phi(s,a)\|_{(\Lambda_h+\lambda I)^{-1}} = \Ocal(n^{-1/2})$ holds for all $(s,a)\in\Scal\times\Acal$ \citep{duan2020minimax} and $\|r_h-\widehat{r}_h\|_\infty = \Ocal(n^{-1/2})$. However, even without strong assumptions such as sufficient coverage, we can still prove a $\Ocal(n^{-1/2})$ suboptimality with pessimistic value iteration. 
\end{remark}

\section{Policy Learning from Dynamic Choices via Pessimistic Value Iteration}
In this section, we describe the pessimistic value iteration algorithm, which minus a penalty function $\Gamma_h: \Scal\times\Acal\rightarrow\R$ from the value function when choosing the best action. Pessimism is achieved when $\Gamma_h$ is a \textit{uncertainty quantifier} for our learned value functions $\{\tilde{V}_h\}_{h\in[H]}$ , i.e. \begin{align}\label{eq:uncertainty-quant}
    \big|\big(\widehat{r}_h + \widetilde{\P}_h\widetilde{V}_{h+1}\big)(s,a) - \big(r_h + \P_h\widetilde{V}_{h+1}\big)(s,a)   \big| \leq \Gamma_h(s,a) \text{ for all }(s,a)\in\Scal\times\Acal 
\end{align}
with high probability.
Then we use  $\{\Gamma_h\}_{h\in[H]}$  as the penalty function for pessimistic
planning, which leads to a conservative estimation of the value function. We formally describe our
planning method in Algorithm \ref{alg:pess-value-iter}.
However, when doing pessimistic value iteration with $\{\widehat{r}_h\}_{h\in[H]}$ learned from human feedback, it is more difficult to design uncertainty quantifiers in \eqref{eq:uncertainty-quant}, since the estimation error from reward learning is inherited in pessimistic planning. In Section \ref{sec:subopt-linear}, we propose an efficient uncertainty quantifier and prove that with pessimistic value iteration, Algorithm \ref{alg:pess-value-iter} can achieve a $\Ocal(n^{-1/2})$ suboptimality gap even without any observable reward signal, which matches current standard results in pessimistic value iteration such as \citep{jin2021pessimism,uehara2021pessimistic, uehara2021representation}.

\begin{algorithm}[H]\label{alg:pess-vi}
   \caption{DCPPO: Pessimistic Value iteration }
     \textbf{Require:} Surrogate reward $\{\widehat{r}_h(s_h,a_h) \}_{h\in[H]}$ learned in Algorithm \ref{alg:reward-learn}, collected dataset $\{(s_h^i,a_h^i)\}_{i\in[n], h\in[H]}$, parameter $\beta$, penalty  .\\
     \textbf{Initialization:} Set $\widetilde{V}_{H+1}(s_{H+1}) = 0$.
   \begin{algorithmic}[1]\label{alg:pess-value-iter}
   	 \FOR{step $h=H, \dots,1$}
   	\STATE Set $\widetilde{\P}_h\widetilde{V}_{h+1}(s_h,a_h) = \operatorname{argmin}_{f} \sum_{i\in[n]} \big(f(s_h^i,a_h^i) - \widetilde{V}_{h+1}(s_{h+1})\big)^2 + \lambda\cdot\rho(f) $.
   	\STATE Construct $\Gamma_h(s_h,a_h)$ based on $\Dcal$.
   	\STATE Set $\widetilde{Q}_h(s_h,a_h) = \min\big\{\widehat{r}_h(s_h,a_h) + \widetilde{P}_h\widetilde{V}_{h+1}(s_h,a_h) - \Gamma_h(s_h,a_h), H-h+1\big\}_{+}$.

   \STATE Set $\widetilde{\pi}_h(\cdot\mid s_h) = \operatorname{argmax} \langle \widetilde{Q}_h(s_h,\cdot), \pi_h(\cdot\mid s_h)\rangle$.
   \STATE Set $\widetilde{V}_h(s_h) = \langle \tilde{Q}_h(s_h,\cdot), \Tilde{\pi}_h(\cdot\mid s_h)\rangle_{\Acal}$.
   	 \ENDFOR\\
    \STATE \textbf{Output:} $\{\widetilde{\pi}_h\}_{h\in[H]}$.
   \end{algorithmic}
\end{algorithm}
\subsection{Suboptimality Gap of Pessimitic Optimal Policy}\label{sec:subopt-linear}
For the linear model class defined in Section \ref{sec:linear MDP}, we assume that we can capture the conditional expectation of value function in the next step with the known feature $\phi$. In formal words, we make the following assumption.
\begin{assumption}[\textbf{Linear MDP}]\label{ass:linear-transition}
    For the underlying MDP, we assume that for every $V_{h+1}: \Scal\rightarrow [0, H-h]$, there exists $u_h \in \R^d$ such that $$
    \expectV(s,a) = \phi(s,a)\cdot u_h
    $$
    for all $(s,a)\in\Scal\times\Acal$. We also assume that $\|u_h\| \leq (H-h+1)\cdot\sqrt{d}$ for all $h\in[H]$.
\end{assumption}
Note that this assumption is directly satisfied by linear MDP class \citep{jin2021pessimism,jin2020provably,yang2019sampleoptimal}. For linear model MDP defined in Section \ref{sec:linear MDP}, it suffices to have the parameter set $\Theta$ being closed under subtraction, i.e. if $x,y\in\Theta$ then $x-y\in\Theta$.
Meanwhile, we construct $\Gamma_h$ in Algorithm \ref{alg:pess-value-iter} based on dataset $\Dcal$ as \begin{equation}\label{eq:penalty-linear}
    \Gamma_h(s,a) = \beta\cdot \big(\phi(s,a)^\top (\Lambda_h+\lambda I)^{-1}\phi(s,a)\big)^{1/2}
\end{equation}
for every $h\in[H]$. Here $\Lambda_h$ is defined in \eqref{eq:ridge-sol}.
To establish suboptimality for Algorithm \ref{alg:pess-value-iter}, we assume that  the trajectory induced by $\pi^*$
is “covered” by $\Dcal$ sufficiently well. 
\begin{assumption}[\textbf{Single-Policy Coverage}]\label{ass:coverage-ass}
    Suppose there exists an absolute constant $c^{\dagger}>0$ such that $$
    \Lambda_h \geq c^{\dagger} \cdot n \cdot \mathbb{E}_{\pi^*}\left[\phi\left(s_h, a_h\right) \phi\left(s_h, a_h\right)^{\top} \right]
    $$ holds with probability at least $1-\delta/2$.
\end{assumption}
We remark that Assumption \ref{ass:coverage-ass} only assumes that human behavior policy can cover the optimal policy and is, therefore, weaker than assuming a well-explored dataset, or sufficient coverage in RL and econometrics literature \citep{duan2020minimax, jin2021pessimism,aguirregabiria2010dynamic}.  With this assumption, we prove the following theorem.
\begin{theorem}[\textbf{Suboptimality Gap for DCPPO}]\label{thm:pess-vi-linear-cover}
Suppose Assumption \ref{ass:identify}, \ref{ass:regular}, \ref{ass:linear-transition},\ref{ass:coverage-ass} holds. With $\lambda = 1$  and $\beta = \Ocal({He^{H}}\cdot|\Acal|\cdot d\sqrt{\log\big({nH}/{\delta}\big)})$, we have (i) $\Gamma_h$ defined in \eqref{eq:penalty-linear} being uncertainty quantifiers, and (ii)
$$
    \operatorname{SubOpt}\big(\{\tilde{\pi}_h\}_{h\in[H]}\big) \leq c\cdot (1+\gamma)|\Acal|d^{3/2}H^2e^Hn^{-1/2}\sqrt{\xi}
$$
holds for Algorithm \ref{alg:pess-value-iter} with probability at least $1-\delta$ , here $\xi = \log(dHn/\delta)$. In particular, if $\operatorname{rank}(\Sigma_h)\leq r$ at each step $h\in[H]$, then $$
\operatorname{SubOpt}\big(\{\tilde{\pi}_h\}_{h\in[H]}\big) \leq c\cdot (1+\gamma)|\Acal|r^{1/2}dH^2e^Hn^{-1/2}\sqrt{\xi},
$$
here $\Sigma_h = \E_{\pi_b}[\phi(s_h,a_h)\phi(s_h,a_h)^\top]$.
\end{theorem}
\begin{proof}
    See Appendix \ref{sec:prove-pess-vi} for detailed proof.
\end{proof}

\paragraph{Remark.} It is worth highlighting that Theorem \ref{thm:pess-vi-linear-cover} nearly matches the standard result for pessimistic offline RL with observable rewards in terms of the dependence on data size and distribution, up to a constant factor of $\Ocal(|\Acal|e^H)$ \citep{jin2020provably,uehara2021pessimistic}, where their suboptimality is of $\tilde{\Ocal}(dH^2n^{-1/2})$.  Therefore, Algorithm \ref{alg:reward-learn} and \ref{alg:pess-value-iter} \textit{almost} matches the suboptimality gap of standard pessimism planning with an observable reward, except for a $\Ocal(e^H)$ factor inherited from reward estimation.



\section{DCPPO for Reproducing Kernel Hilbert Space}\label{sec: RKHS-case}

In this section, we assume the model class $\Mcal = \{\Mcal_h\}_{h\in[H]}$ are subsets of a Reproducing Kernel Hilbert Space (RKHS). For notations simplicity, we let $z=(s, a)$ denote the state-action pair and denote $\mathcal{Z}=\mathcal{S} \times \mathcal{A}$ for any $h \in[H]$. We view $\mathcal{Z}$ as a compact subset of $\mathbb{R}^d$ where the dimension $d$ is fixed. Let $\mathcal{H}$ be an RKHS of functions on $\mathcal{Z}$ with kernel function $K: \mathcal{Z} \times \mathcal{Z} \rightarrow \mathbb{R}$, inner product $\langle\cdot, \cdot\rangle: \mathcal{H} \times \mathcal{H} \rightarrow \mathbb{R}$ and RKHS norm $\|\cdot\|_{\mathcal{H}}: \mathcal{H} \rightarrow \mathbb{R}$. By definition of RKHS, there exists a feature mapping $\phi: \mathcal{Z} \rightarrow \mathcal{H}$ such that $f(z)=\langle f, \phi(z)\rangle_{\mathcal{H}}$ for all $f \in \mathcal{H}$ and all $z \in \mathcal{Z}$. Also, the kernel function admits the feature representation $K(x, y)=\langle\phi(x), \phi(y)\rangle_{\mathcal{H}}$ for any $x, y \in \mathcal{H}$. We assume that the kernel function is uniformly bounded as $\sup _{z \in \mathcal{Z}} K(z, z)<\infty$. For notation simplicity, we assume that the discount factor $\gamma = 1$.
Let $\mathcal{L}^2(\mathcal{Z})$ be the space of square-integrable functions on $\mathcal{Z}$  and let $\langle\cdot, \cdot\rangle_{\mathcal{L}^2}$ be the inner product for $\mathcal{L}^2(\mathcal{Z})$. We define the Mercer operater $T_K: \mathcal{L}^2(\mathcal{Z}) \rightarrow \mathcal{L}^2(\mathcal{Z})$,
\begin{equation}\label{eq:mercer-integral}
    T_K f(z)=\int_{\mathcal{Z}} K\left(z, z^{\prime}\right) \cdot f\left(z^{\prime}\right) \mathrm{d} z^{\prime}, \quad \forall f \in \mathcal{L}^2(\mathcal{Z}).
\end{equation}
By Mercer's Theorem \citep{Steinwart2008SupportVM},  there exists a countable and non-increasing sequence of non-negative eigenvalues $\left\{\sigma_i\right\}_{i \geq 1}$ for the operator $T_K$, with associated orthogonal eigenfunctions $\left\{\psi_i\right\}_{i>1}$ .   
In what follows, we assume the eigenvalue of the integral operator defined in \ref{eq:mercer-integral} has a certain decay condition.
\begin{assumption}[Eigenvalue Decay of $\mathcal{H}$]\label{ass:eig-decay}
Let $\left\{\sigma_j\right\}_{j \geq 1}$ be the eigenvalues induced by the integral opretaor $T_K$ defined in Equation \eqref{eq:mercer-integral} and $\left\{\psi_j\right\}_{j \geq 1}$ be the associated eigenfunctions. We assume that $\left\{\sigma_j\right\}_{j \geq 1}$ satisfies one of the following conditions for some constant $\mu>0$.
\begin{itemize}
    \item[(i)] $\mu$-finite spectrum: $\sigma_j=0$ for all $j>\mu$, where $\mu$ is a positive integer.
    \item[(ii)]  $\mu$-exponential decay: there exists some constants $C_1, C_2>0, \tau \in[0,1 / 2)$ and $C_\psi>0$ such that $\sigma_j \leq C_1 \cdot \exp \left(-C_2 \cdot j^\mu\right)$ and $\sup _{z \in \mathcal{Z}} \sigma_j^\tau \cdot\left|\psi_j(z)\right| \leq C_\psi$ for all $j \geq 1$.
    \item[(iii)]  $\mu$-polynomial decay: there exists some constants $C_1>0, \tau \in[0,1 / 2)$ and $C_\psi>0$ such that $\sigma_j \leq C_1 \cdot j^{-\mu}$ and $\sup _{z \in \mathcal{Z}} \sigma_j^\tau \cdot\left|\psi_j(z)\right| \leq C_\psi$ for all $j \geq 1$, where $\mu>1$.
\end{itemize}
\end{assumption}
For a detailed discussion of eigenvalue decay in RKHS, we refer the readers to Section 4.1 of \cite{yang2020provably}. 
\subsection{Gurantee for RKHS}
In RKHS case, our first step MLE in \eqref{eq:mle} turns into a kernel logistic regression, \begin{equation}\label{eq:log-reg}
    \bar{Q}_h=\operatorname{argmin}_{Q\in \Hcal}\frac{1}{n} \sum_{i=1}^nQ(\thissai)- \log\bigg(\sum_{a'\in\Acal} \exp(Q(s,a'))\bigg).
\end{equation}
Line \ref{line:bell-mse} in Algorithm \ref{alg:reward-learn} now turns into a kernel ridge regression for the Mean Squared Bellman Error, with $\rho(f)$ being $\|f\|_\Hcal^2$, \begin{equation}\label{eq:kernel-ridge-reg}
    \widehat{r}_h = \operatorname{argmin}_{r\in\Hcal}  \bigg\{ \sum_{i=1}^n \bigg(r(\thissai) + \gamma\cdot \learnednextv(\nextstatei) - \learnedq(\thissai)\bigg)^2 + \lambda \|r\|^2_\Hcal \bigg\}.
\end{equation}
  Following Representer's Theorem \citep{Steinwart2008SupportVM}, we have the following closed form solution $$
\widehat{r}_h(z) = k_h(z)^\top (K_h+\lambda\cdot I)^{-1}y_h,
$$
where we define the Gram matrix $K_h \in \mathbb{R}^{n \times n}$ and the function $k_h: \mathcal{Z} \rightarrow \mathbb{R}^{n}$ as
\begin{align}\label{eq:gram-mat}
    K_h=\big[K\big(z_h^i, z_h^{i^{\prime}}\big)\big]_{i, i^{\prime} \in [n]} \in \mathbb{R}^{n \times n}, \quad k_h(z)=\big[K\big(z_h^i, z\big)\big]_{i \in [n]}^{\top} \in \mathbb{R}^{n},
\end{align}
and the entry of the response vector $y_h \in \mathbb{R}^{n}$ corresponding to $i \in [n]$ is
$$
\left[y_h\right]_i=\widehat{Q}_h(\thissai)-\gamma\cdot\widehat{V}_{h+1}\left(s_{h+1}^i\right).
$$
Meanwhile, we also construct the uncertainty quantifier $\Gamma_h$ in Algorithm \ref{alg:pess-value-iter},\begin{equation}\label{eq:penalty-rkhs}
    \Gamma_h(z) = \beta\cdot\lambda^{-1/2}\cdot\big(K(z, z)-k_h(z)^{\top}\left(K_h+\lambda I\right)^{-1} k_h(z)\big)^{1 / 2}
\end{equation}
for all $z\in\Zcal$.
Parallel to Assumption \ref{ass:linear-transition}, we impose the following structural assumption for the kernel setting.
\begin{assumption}\label{ass:rkhs-retain}
     Let $R_r>0$ be some fixed constant and we define function class $\mathcal{Q}=\left\{f \in \mathcal{H}:\|f\|_{\mathcal{H}} \leq\right.$ $\left.HR_r \right\}$. We assume that $\P_hV_{h+1} \in \mathcal{Q}$ for any $V_{h+1}:\Scal\rightarrow[0,H]$. We also assume that $\|r\|_{\Hcal}\leq R_r$ for some constant $R_r>0$. We set the model class $\Mcal_h = \Qcal$ for all $h\in[H]$.
\end{assumption}
The above assumption states that the Bellman operator maps any bounded function into a bounded RKHS-norm ball, and holds for the special case of linear MDP \cite{jin2021pessimism}.

Besides the closeness assumption on the Bellman operator, we also define the maximal information gain \citep{srinivas2009gaussian} as a characterization of the complexity of $\mathcal{H}$ :
\begin{equation}\label{maximal information
gain}
    G(n, \lambda)=\sup \left\{1 / 2 \cdot \log \operatorname{det}\left(I+K_{\mathcal{C}} / \lambda\right): \mathcal{C} \subset \mathcal{Z},|\mathcal{C}| \leq n\right\}
\end{equation}
Here $K_{\mathcal{C}}$ is the Gram matrix for the set $\mathcal{C}$, defined similarly as Equation \eqref{eq:gram-mat}. In particular, when $\mathcal{H}$ has $\mu$-finite spectrum, $G(n, \lambda)=\mathcal{O}(\mu \cdot \log n)$ recovers the dimensionality of the linear space up to a logarithmic factor.

We are now ready to present our result for reward estimation in the RKHS case. 
\begin{theorem}[\textbf{Reward Estimation for RKHS}]\label{thm:reward-est-rkhs}
    For Algorithm \ref{alg:reward-learn} and \ref{alg:pess-value-iter}, with probability at least $1-\delta$, we have the following estimations of our reward function for all $z \in \Zcal \times \Acal$ and $\lambda > 1$,
    \begin{align*}
        &|r_h(z) - \widehat{r}_h(z) | \\
        &\quad \leq \|\phi(z)\|_{(\Lambda_h+\lambda \Ical_{\Hcal})^{-1}}\cdot \Ocal\bigg(H^2\cdot G\big(n, 1+1/n\big)+\lambda\cdot R_{r}^2+\zeta^2\bigg)^{1/2},
    \end{align*}
    where $$\zeta = \Ocal\bigg(\dimsamp\sqrt{ \log\big(H\cdot N(\Qcal, \|\cdot\|_{\infty},1/n )/\delta\big)}\cdot He^H \bigg).$$
    Here $\dimsamp$ is the sampling effective dimension. 
\end{theorem}
\begin{proof}
    See Appendix \ref{sec:reward-est-rkhs} for detailed proof.
\end{proof}
Here we use the notation $\|\phi(z)\|_{(\Lambda_h+\lambda \Ical_{\Hcal})^{-1}} = \langle \phi(z), (\Lambda_h+\lambda \Ical_{\Hcal})^{-1}\phi(z)\rangle_\Hcal$, where we define $$\Lambda_h = \sum_{i=1}^n \phi(z_h^i)\phi(z_h^i)^\top.$$
With the guarantee of Theorem \ref{thm:reward-est-rkhs}, Theorem \ref{thm:pessim-vi-rkhs} establishes the concrete suboptimality of DCPPO under various eigenvalue decay conditions.
\begin{theorem}[\textbf{Suboptimality Gap for RKHS}]\label{thm:pessim-vi-rkhs}
    Suppose that Assumption \ref{ass:eig-decay} holds. For $\mu$-polynomial decay, we further assume $\mu(1-2\tau)>1$. For Algorithm \ref{alg:reward-learn} and \ref{alg:pess-value-iter}, we set $$
    \lambda= \begin{cases}C \cdot \mu \cdot \log (n / \delta) & \mu \text {-finite spectrum, } \\ C \cdot\log (n / \delta)^{1+1 / \mu} & \mu \text {-exponential decay, } \\ C \cdot(n / H)^{\frac{2}{\mu(1-2 \tau)-1}} \cdot \log (n / \delta) & \mu \text {-polynomial decay, }\end{cases}
    $$
    and $$
    \beta= \begin{cases}C'' \cdot H\cdot\bigg\{\sqrt{\lambda} R_r +\dimsamp e^H |\Acal|\cdot\log (n R_rH/ \delta)^{1 / 2+1 /(2 \mu)}\bigg\} & \mu \text {-finite spectrum, } \\ C'' \cdot H\cdot\left\{\sqrt{\lambda} R_r +\dimsamp e^H |\Acal|\cdot\log (n R_rH/ \delta)^{1 / 2+1 /(2 \mu)}\right\} & \mu \text {-exponential decay, } \\ C'' \cdot H\cdot\left\{\sqrt{\lambda}R_r + \dimsamp e^H |\Acal|\cdot(nR_r)^{\kappa^*} \cdot \sqrt{\log (nR_rH/\delta)}\right\} & \mu \text {-polynomial decay. }\end{cases}
    $$
    Here $C > 0$ is an absolute constant that does not depend on $n$ or $H$. Then with probability at least $1-\delta$,  it holds that (i)  $\Gamma_h$ set in \eqref{eq:penalty-rkhs} being uncertainty quantifiers, and (ii)
    $$
    \operatorname{SubOpt}(\{\widetilde{\pi}_h\}_{h\in[H]}) \leq \begin{cases}C' \cdot \tilde{d} \cdot He^H|\Acal| \sqrt{\mu \cdot \log (nR_rH / \delta)}\} & \mu \text {-finite spectrum, } \\ C' \cdot \tilde{d}\cdot He^H |\Acal|\cdot\sqrt{(\log (n R_rH) / \delta)^{1+1 / \mu}} & \mu \text {-exponential decay, } \\ C' \cdot \tilde{d} \cdot He^H|\Acal| \cdot (nR_r)^{\kappa^*}  \cdot \sqrt{\log (nR_r H / \delta)} & \mu \text {-polynomial decay.}\end{cases}
    $$ 
Here $C, C',C''$ are absolute constants irrelevant to $n$ and $H$ and $$
\tilde{d} = \dimpop\cdot{\dimsamp}, \quad
    \kappa^*=\frac{d+1}{2(\mu+d)}+\frac{1}{\mu(1-2 \tau)-1}.
    $$
Here $\dimpop$ is the population effective dimension, which measures the "coverage" of the human behavior $\pi_b$ for the optimal policy $\pi^*$. 
\end{theorem}
\begin{proof}
See Appendix \ref{sec:prove-pessi-vi-rkhs} for detailed proof.
\end{proof}
For simplicity of notation, we delay the formal definition of $\dimsamp$ and $\dimpop$ to the appendix. Parallel to the linear setting, Theorem \ref{thm:pessim-vi-rkhs} demonstrates the performance of our method in terms of effective dimension. If the behavior policy is close to the optimal policy and the RKHS satisfies Assumption \ref{ass:eig-decay}, $\dimpop = \Ocal(H^{3/2}n^{-1/2})$ and $\dimsamp$ remains in constant level. In this case  suboptimality is of order $\Ocal(n^{-1/2})$ for $\mu$-finite spectrum and $\mu$-exponential decay, while for $\mu$-polynomial decay we obtain a rate of $\Ocal(n^{\kappa^*-1/2})$. This also matches the results in standard pessimistic planning under RKHS case \citep{jin2021pessimism}, where the reward is observable.


\section{Conclusion}
In this paper, we have developed a provably efficient online algorithm, \underline{D}ynamic-\underline{C}hoice-\underline{P}essimistic-\underline{P}olicy-\underline{O}ptimization (DCPPO) for RLHF under dynamic discrete choice model. By maximizing log-likelihood function of the Q-value function and minimizing Bellman mean square error for the reward, our algorithm learns the unobservable reward, and the optimal policy following the principle of pessimism. We prove that our algorithm is efficient in
sample complexity for linear model MDP and RKHS model class.  To the best of our knowledge, this
is the first provably efficient algorithm for offline RLHF under the dynamic discrete choice model.

\bibliography{citation}
\bibliographystyle{plainnat}
\appendix
\newpage

\section{Proof for Theorem \ref{thm:emp-mle-guarantee}}\label{app:proof-policy-recover}
Theorem \ref{thm:emp-mle-guarantee} can be regarded as an MLE guarantee for dataset distribution.
Our proof for Theorem \ref{thm:emp-mle-guarantee} lies in two steps: (i) We prove an MLE guarantee in population distribution, i.e. when $s_h$ is sampled by the behavior policy $\behpolicy$, the estimation error can be bounded in expectation; (ii) With a concentration approach, we transfer the expectation bound to a bound on a realized dataset.
First, for MLE with an identifiable $\behaviorq\in\Mcal_h$, we have the following guarantee:
\begin{lemma}[MLE distribution bound]\label{thm:mle-guarantee}
For $\learnedq$ estimated by \eqref{eq:mle}, we have $$
\E_{s_h\sim\pi_b}\big[\|\learnedpolicy(\cdot\mid\thisstate) - \behpolicy(\cdot\mid \thisstate)\|_1^2\big] \leq c\cdot \frac{\log\big(H\cdot N(\Mcal, \|\cdot\|_{\infty}, 1/n )/\delta\big)}{n}
$$
with probability at least $1-\delta$. Here $c, c' >0$ are two absolute constants, and $\delta$ measures the confidence in the estimation.
\end{lemma}
\begin{proof}
For all $h\in[H]$, define  $$
 \Pi_h = \big\{
 \pi_Q(a\mid s) = {\exp({Q(s,a)})}/{\sum_{a'\in\Acal} \exp({Q(s,a')})} \text{ for some } Q\in\Mcal_h
 \big\}.$$
Let $\Ncal_{[]}(\Pi_h,\|\cdot\|_\infty,1/n)$ be the smallest $1/n$-upper bracketing of $\Pi_h$. And $|\Ncal_{[]}(\Pi_h,\|\cdot\|_\infty,1/n)| = \bracketing(\Pi_h,\|\cdot\|_\infty,1/n)$, where $\bracketing(\Pi_h,\|\cdot\|_\infty,1/n)$ is the  bracketing number of $\Pi_h$.
First, we prove that $$
\E_{s_h\sim\pi_b}\big[\|\learnedpolicy(\cdot\mid\thisstate) - \behpolicy(\cdot\mid \thisstate)\|_1^2\big] \leq \Ocal\bigg( \frac{\log\big(H\cdot \bracketing(\Pi_h, \|\cdot\|_{\infty}, 1/n )/\delta\big)}{n}\bigg)
$$
with probability at least $1-\delta$. By MLE guarantee, we have $$
\frac{1}{n}\sum_{i=1}^n\log\bigg(\frac{\widehat{\pi}_h(a_h^t\mid s_h^t)}{\pi_{b,h}(a_h^t\mid s_h^t)}\bigg) \geq 0,
$$
by Markov's inequality and Boole’s inequality, it holds with probability at least $1-\delta$ that for all $\bar{\pi}\in\Ncal_{[]}(\Pi,\|\cdot\|_\infty,1/n)$, we have $$
\sum_{i=1}^n \frac{1}{2}\log\bigg(\frac{\bar{\pi}(a_h^t\mid s_h^t)}{\pi_{b,h}(a_h^t\mid s_h^t)}\bigg) \leq n\log\bigg(\E_{\pi_b}\bigg[\exp\bigg(\frac{1}{2}\log\bigg(\frac{\bar{\pi}(\cdot\mid \cdot)}{\pi_{b,h}(\cdot\mid \cdot)}\bigg) \bigg)\bigg]\bigg)+\log\bigg(\frac{N_{[]}(\Pi,\|\cdot\|_\infty, 1/n)}{\delta}\bigg),
$$
specify $\bar{\pi}$ to be the upper bracket of $\widehat{\pi}_h$,  we have \begin{align*}
    0 &\leq n\log\bigg(\E_{\pi_b}\bigg[\exp\bigg(\frac{1}{2}\log\bigg(\frac{\bar{\pi}(\cdot\mid \cdot)}{\pi_{b,h}(\cdot\mid \cdot)}\bigg) \bigg)\bigg]\bigg)+\log\bigg(\frac{N_{[]}(\Pi,\|\cdot\|_\infty, 1/n)}{\delta}\bigg)\\
    &\leq n\cdot \log\bigg(\E_{\pi_b}\bigg[\sqrt{\frac{\bar{\pi}(\cdot\mid \cdot)}{\pi_{b,h}(\cdot\mid \cdot)}}\bigg]\bigg)+\log\bigg(\frac{N_{[]}(\Pi,\|\cdot\|_\infty, 1/n)}{\delta}\bigg)\\
    &= n\cdot\log\bigg(\E_{s_h\sim\pi_b} \bigg[\sum_{a\in\Acal} \sqrt{\bar{\pi}(a\mid s_h)\cdot \pi_{b,h}(a\mid s_h)}\bigg] \bigg) + \log\bigg(\frac{\bracketing(\Pi,\|\cdot\|_\infty,1/n)}{\delta}\bigg),
\end{align*}
Here $\E_{s_h\sim\pi_b}$ means $s_h$ is simulated by the policy $\pi_b$. Utilizing the $\log x \leq x-1$, we have $$
1- \E_{\pi_b}\bigg[\sum_{a\in\Acal} \sqrt{\bar{\pi}(a\mid s_h)\cdot \pi_{b,h}(a\mid s_h)}\bigg] \leq \frac{1}{n} \log\bigg(\frac{\bracketing(\Pi,\|\cdot\|_\infty,1/n)}{\delta}\bigg).
$$
Therefore we can bound the Hellinger distance between $\pi_{b,h}$ and $\bar{\pi}$,
\begin{align}\label{hell-dist-1}
    h(\bar{\pi},\pi_{b,h}) &= \E_{s_h\sim\pi_b}\bigg[\sum_{a\in\Acal}\big(\bar{\pi}(a\mid s_h)^{1/2} - \pi_{b,h}(a\mid s_h)^{1/2} \big)^2\bigg]\\
    &\leq 2\bigg(1-\sum_{a\in\Acal} \sqrt{\bar{\pi}(a\mid s_h)\cdot \pi_{b,h}(a\mid s_h)}\bigg) + \frac{1}{n}\\
    &\leq \frac{2}{n}\log\bigg(\frac{N_{[]}(\Pi,\|\cdot\|_\infty,1/n)}{\delta}\bigg)+\frac{1}{n},
\end{align}
here the second inequality comes from the fact that $\bar{\pi}$ is a upper bracketing of $\Pi$. Moreover, it is easy to verify that \begin{align}\label{eq:hell-dis-2}
\E_{s_h\sim\pi_b}\big[\sum_{a\in\Acal}((\bar{\pi}(a\mid s_h)^{1/2} + \pi_{b,h}(a\mid s_h)^{1/2} ))^2\big] &\leq 2\E_{s_h\sim\pi_b}\big[\sum_{a\in\Acal}(\bar{\pi}(a\mid s_h) + \pi_{b,h}(a\mid s_h))\big] \\
&\leq \frac{2}{n} +4,
\end{align}
where the second inequality comes from the fact that $\bar{\pi}$ is the $1/n$-upper bracket of a probability distribution.
Combining the \eqref{hell-dist-1} and \eqref{eq:hell-dis-2}, by Cauchy-Schwarz inequality, we have $$
\E_{s_h\sim\pi_b}\big[\|\bar{\pi}(\cdot\mid\thisstate) - \behpolicy(\cdot\mid \thisstate)\|_1^2\big] \leq \frac{15}{n}\cdot \log\bigg(\frac{N_{[]}(\Pi,\|\cdot\|_\infty,1/n)}{\delta}\bigg).
$$
Meanwhile, \begin{align*}
    &\|\bar{\pi}(\cdot\mid\thisstate) - \behpolicy(\cdot\mid \thisstate)\|_1^2 - \|\widehat{\pi}_h(\cdot\mid\thisstate) - \behpolicy(\cdot\mid \thisstate)\|_1^2\\
    &\qquad\qquad\leq \big(\sum_{a\in\Acal}|\bar{\pi}(\cdot\mid\thisstate) - \behpolicy(\cdot\mid \thisstate)| + \sum_{a\in\Acal}|\widehat{\pi}_h(\cdot\mid\thisstate) - \behpolicy(\cdot\mid \thisstate)|\big)\\
    &\qquad\qquad\qquad\cdot\big(\sum_{a\in\Acal}|\bar{\pi}(\cdot\mid\thisstate) - \behpolicy(\cdot\mid \thisstate)| - \sum_{a\in\Acal}|\widehat{\pi}_h(\cdot\mid\thisstate) - \behpolicy(\cdot\mid \thisstate)|\big)\\
    &\qquad\qquad\leq(4+\frac{1}{n})\cdot\frac{1}{n},
\end{align*}
therefore we have \begin{align*}   \E_{s_h\sim\pi_b}\big[\|\widehat{\pi}_h(\cdot\mid\thisstate) - \behpolicy(\cdot\mid \thisstate)\|_1^2\big] &\leq \frac{20}{n}\cdot\log\bigg(\frac{N_{[]}(\Pi_h,\|\cdot\|_\infty,1/n)}{\delta}\bigg).
\end{align*}
Next, we bound $
N_{[]}(\Pi_h,\|\cdot\|_\infty,1/n) $ by  $N(\Mcal_h,\|\cdot\|_\infty,1/4n) $.
For all $h\in[H]$, recall the definition  $$
 \Pi_h = \big\{
 \pi_Q(a\mid s) = {\exp({Q(s,a)})}/{\sum_{a'\in\Acal} \exp({Q(s,a')})} \text{ for some } Q\in\Mcal_h
 \big\},
 $$
 it is easy to check that $$
 |\pi_Q(a\mid s) - \pi_{Q'}(a\mid s)| \leq 2\cdot \|Q - Q'\|_{\infty} , \forall (s,a)\in\Scal\times\Acal.
 $$
 Recall that $N(\Mcal_h,\|\cdot\|_\infty,1/n)$ is the covering number for model class $\Mcal_h$. Using Lemma \ref{lem:cover-bracket}, we have \begin{equation}\label{eq:cover-brack}
    N_{[]}(\Pi_h,\|\cdot\|_\infty,1/n)\leq N(\Mcal_h,\|\cdot\|_\infty,1/4n)
\end{equation}
always hold for all $h\in[H]$.
Therefore we have \begin{align*}
\E_{s_h\sim\pi_b}\big[\|\widehat{\pi}_h(\cdot\mid\thisstate) - \behpolicy(\cdot\mid \thisstate)\|_1^2\big] \leq \Ocal\bigg(\frac{\log(H\cdot N(\Mcal_h,\|\cdot\|_\infty,1/n)/\delta)}{n}\bigg)
\end{align*}
holds for  $h\in[H]$ with probability $1-\delta/H$.
 Taking union bound on $h\in[H]$ and we conclude the proof for Lemma \ref{thm:mle-guarantee}.
 \end{proof}

\subsection{Proof for Theorem \ref{thm:emp-mle-guarantee}}
From Lemma \ref{thm:mle-guarantee},  we have the following generalization bound: with probability $1-\delta$, $$
\E_{s_h\sim\pi_b}\big[\|\widehat{\pi}_h(\cdot\mid\thisstate) - \behpolicy(\cdot\mid \thisstate)\|_1^2\big]\leq\Ocal\bigg(\frac{\log(H\cdot N(\Mcal_h,\|\cdot\|_\infty,1/n)/\delta)}{n}\bigg)
$$ for all $h\in[H]$. We now condition on this event.
Letting $$
A(\widehat{\pi}_h):= \big|\E_{s_h\sim\pi_b}\big[\|\widehat{\pi}_h(\cdot\mid\thisstate) - \behpolicy(\cdot\mid \thisstate)\|_1^2\big] - \E_{\Dcal_h}\big[\|\widehat{\pi}_h(\cdot\mid\thisstate) - \behpolicy(\cdot\mid \thisstate)\|_1^2\big]\big|.
$$
With probability $1-\delta$, from Bernstein’s inequality, we also have \begin{align*}
    A(\widehat{\pi}_h) &\leq \Ocal\bigg(\frac{\log(H/\delta)}{n} + \sqrt{\frac{\operatorname{Var}_{s_h\sim\pi_{b}}[\|\widehat{\pi}_h(\cdot\mid\thisstate) - \behpolicy(\cdot\mid \thisstate)\|_1^2]\log(H/\delta)}{n}}\bigg)\\
    &\leq \Ocal\bigg(\frac{\log(H/\delta)}{n} + \sqrt{\frac{\E_{s_h\sim\pi_{b}}[\|\widehat{\pi}_h(\cdot\mid\thisstate) - \behpolicy(\cdot\mid \thisstate)\|_1^2]\log(H/\delta)}{n}} \bigg)\\
    &\leq \Ocal\bigg(\frac{\log(H\cdot N(\Mcal_h,
    \|\cdot\|_\infty,1/n)/\delta)}{n}\bigg).
\end{align*}
holds for all $h\in[H]$ with probability at least $1-\delta$, and therefore we have $$
\E_{\Dcal_h}\big[\|\widehat{\pi}_h(\cdot\mid\thisstate) - \behpolicy(\cdot\mid \thisstate)\|_1^2\big] \leq \Ocal\bigg(\frac{\log(H\cdot N(\Mcal_h,
    \|\cdot\|_\infty,1/n)/\delta)}{n}\bigg),
$$
i.e. the error of estimating $\pi_{b,h}$ decreases in scale $\tilde{O}(1/n)$ on the dataset. Recall that $$
\learnedpolicy(a\mid s) = \frac{\exp(\learnedq(s,a))}{\sum_{a'\in\Acal} \exp(\learnedq(s_h,a')}
$$
and $$
    \pi_{b,h}(a\mid s) = \frac{\exp(\behaviorq(s,a))}{\sum_{a'\in\Acal} \exp(\behaviorq(s,a'))}.
$$
Also we have $\widehat{Q}_h(s,a_0) = \behaviorq(s,a_0) = 0$ and $\behaviorq\in[0,H]$ by Assumption \ref{ass:identify} and definition of $\Mcal_h$. Therefore, we have $$
\big|\learnedq(s,a) - \behaviorq(s,a)\big| = \bigg|\log\bigg(\frac{\learnedpolicy(a\mid s)}{\behpolicy(a\mid s)}\bigg)  - \log\bigg(\frac{\learnedpolicy(a_0\mid s)}{\behpolicy(a_0\mid s)}\bigg) \bigg|.
$$
Utilizing $\ln(x/y) \leq x/y -1 $ for $x,y >0$, and $\pi_h(a\mid s)\in [e^{-H},1]$,  we have $$
\big|\learnedq(s,a) - \behaviorq(s,a)\big| \leq e^{H}\cdot\bigg(|\behpolicy(a\mid s) - \learnedpolicy(a\mid s) | + |\behpolicy(a_0\mid s) - \learnedpolicy(a_0\mid s) |\bigg),
$$
and by taking summation over $a\in\Acal$, we have $$
\E_{s_h\in\Dcal_h}\big[\|\learnedq(\thisstate,\cdot) - \behaviorq(\thisstate, \cdot)\|_1^2\big] \leq c'\cdot \frac{H^2 e^{2H}\cdot|\Acal|^2 \cdot \log\big(H\cdot N(\Mcal, \|\cdot\|_{\infty}, 1/n )/\delta\big)}{n},
$$
and we complete our proof.
\section{Proof for Theorem \ref{thm:reward-est}}\label{sec:prove-reward-est}
Recall that in \eqref{eq:ridge-sol}, we have $$
    \widehat{w}_h = (\Lambda_h + \lambda I)^{-1} \bigg( \sum_{i=1}^n \phi(\thissai)\big(\learnedq(\thissai) - \gamma\cdot\learnednextv(\nextstatei)\big)\bigg)
$$
where $$
\Lambda_h = \sum_{i=1}^n \phi(\thissai)\phi(\thissai)^\top.
$$
By Assumption \ref{ass:funct-approx}, there exists $w_h\in\R^d$ such that $r_h(s,a) = \phi(s,a)\cdot w_h$. By our construction for $\widehat{r}_h$ in Algorithm \ref{alg:reward-learn}, we therefore have \begin{align*}
    |r_h(s,a) - \widehat{r}_h(s,a)| &= |\phi(s,a) (w_h - \widehat{w}_h)|\\
    &= \bigg|\phi(s,a)\big(\Lambda_h + \lambda I\big)^{-1}\bigg(\lambda\cdot w_h+\sum_{i=1}^n \phi(\thissai)\big(\learnedq(s_h^i,a_h^i) - \gamma\cdot \widehat{V}_{h+1}(s_{h+1}^i) - r_h(\thissai)\big)\bigg)\bigg|\\
    &\leq \underbrace{\lambda\cdot|\phi(s,a)\big(\Lambda_h+\lambda I\big)^{-1}w_h|}_{\text{(i)}} \\
    &\qquad+ \underbrace{\bigg|\phi(s,a)\big(\Lambda_h+\lambda I\big)^{-1}\bigg(\sum_{i=1}^n \big(\learnedq(\thissai) - \gamma\cdot \widehat{V}_h(s_{h+1}^i) - r_h(\thissai)\big)\bigg)\bigg|}_{\text{(ii)}},
\end{align*}
For (i), we have $$
\text{(i)}\leq \lambda\cdot\|\phi(s,a)\|_{(\Lambda_h + \lambda I)^{-1}}\cdot \|w_h\|_{(\Lambda_h + \lambda I)^{-1}},
$$
by Cauchy-Schwarz inequality and by $\Lambda_h$ being semi-positive definite and $\|w_h\|_2\leq\sqrt{d}$, we have \begin{equation}\label{eq:bound(i)}
\text{(i)}\leq \lambda \cdot \sqrt{d/\lambda}\cdot\|\phi(s,a)\|_{(\Lambda_h + \lambda I)^{-1}} = \sqrt{\lambda d}\cdot \|\phi(s,a)\|_{(\Lambda_h + \lambda I)^{-1}},
\end{equation}

and $$
\text{(ii)} \leq \|\phi(s,a)\|_{(\Lambda_h + \lambda I)^{-1}} \underbrace{\cdot \bigg\|\sum_{i=1}^n \phi(\thissai)\big(\learnedq(\thissai) - \gamma\cdot \widehat{V}_{h+1}(\nextstatei) - r_h(\thissai)\big)\bigg\|_{(\Lambda_h+\lambda I)^{-1}}}_{\text{(iii)}} .
$$
Recall that in \eqref{eq:discounted-bellman}, we have the following Bellman equation hold for all $(\thissa)\in\Scal\times\Acal$,
$$
r_h(s_h,a_h) + \gamma\cdot \P_h V_{h+1}^{\pi_b,\gamma}(\thissa) = \behaviorq(\thissa),
$$
substitute this into (iii), and we have  \begin{align}\label{ineq:bellman-transform}
\text{(iii)} &= \bigg\|\sum_{i=1}^n \phi(\thissai)\bigg(\big(\learnedq(\thissai)-\behaviorq(\thissai)\big) - \gamma\cdot \big(\widehat{V}_{h+1}(\nextstatei) - \P_hV^{\pi_b,\gamma}_{h+1}(\nextstatei)\big)\bigg)\bigg\|_{(\Lambda_h+\lambda I)^{-1}}\nonumber\\
&\leq \underbrace{\bigg\|\sum_{i=1}^n \phi(\thissai)\bigg(\big(\learnedq(\thissai)-\behaviorq(\thissai)\big) \bigg)\bigg\|_{(\Lambda_h+\lambda I)^{-1}}}_{\text{(iv)}}\nonumber\\
&\qquad+ \gamma\cdot\underbrace{\bigg\|\sum_{i=1}^n \phi(\thissai)\bigg( \big(\widehat{V}_{h+1}(\nextstatei) - V^{\pi_b,\gamma}_{h+1}(\nextstatei)\big)\bigg)\bigg\|_{(\Lambda_h+\lambda I)^{-1}}}_{\text{(v)}}\nonumber\\
&\qquad+\gamma\cdot\underbrace{\bigg\|\sum_{i=1}^n \phi(\thissai)\bigg( \big(\P_hV^{\pi_b,\gamma}_{h+1}(\thissai) - V^{\pi_b,\gamma}_{h+1}(\nextstatei)\big)\bigg)\bigg\|_{(\Lambda_h+\lambda I)^{-1}}}_{\text{(vi)}},
\end{align}
First, we bound (iv) and (v). By Theorem \ref{thm:emp-mle-guarantee}, we have $$
\E_{\Dcal_h}\big[\|\learnedpolicy(\cdot\mid\thisstate) - \behpolicy(\cdot\mid \thisstate)\|_1^2\big] \leq \Ocal\bigg( \frac{\log\big(H\cdot N(\Mcal_h, \|\cdot\|_{\infty}, 1/n )/\delta\big)}{n}\bigg)
    $$ and 
    \begin{equation}\label{eq:bound-Q}
\E_{\Dcal_h}\big[\|\learnedq(\thisstate,\cdot) - \behaviorq(\thisstate, \cdot)\|_1^2\big] \leq \Ocal\bigg( \frac{H^2 e^{2H} \cdot |\Acal|^2\cdot \log\big(H\cdot N(\Mcal_h, \|\cdot\|_{\infty}, 1/n )/\delta\big)}{n}\bigg)
\end{equation}
hold for every $h\in[H]$ with probability at least $1-\delta/2$. By  $\widehat{V}_h(s) = \langle \widehat{Q}_h(s,\cdot), \learnedpolicy(\cdot\mid s)\rangle_{\Acal}$ for every $s_{h+1}\in\Scal$  , we have \begin{equation}\label{eq:bound-v}
\E_{\Dcal_h}\big[|\widehat{V}_{h+1}(s_{h+1}) - V^{\pi_b,\gamma}_{h+1}(s_{h+1})|^2\big]\leq \Ocal\bigg( \frac{H^2 e^{2H} \cdot |\Acal|^2\cdot \log\big(H\cdot N(\Mcal_h, \|\cdot\|_{\infty}, 1/n )/\delta\big)}{n}\bigg)
\end{equation}
for all $h\in[H]$ simultaneously.
In the following proof, we will condition on these events.
For notation simplicity, we define two function $f:\Xcal\rightarrow \R$ and $\widehat{f}:\Xcal\rightarrow \R$ for each $h\in[H]$, and dataset $\{x_i\}_{i\in[n]}$. We consider two cases:  (1)  $\widehat{f}_h = \learnedq$, $f_h = \behaviorq$, and  $x_i = (\thissai)$, $ \Xcal = \Scal\times\Acal,$  (2) $\widehat{f}_h = \learnednextv$, $f_h = V_{h+1}^{\pi_b,\gamma}$, and $x_i= \nextstatei$ , $\Xcal = \Scal$.  To bound (iv) and (v), we only need to uniformly bound \begin{equation}\label{eq:bound-target}
\bigg\|\sum_{i=1}^n \phi(\thissai) \big(f_h(x_i) - \hat{f}_h(x_i)\big)\bigg\|_{(\Lambda_h+\lambda I)^{-1}}.
\end{equation}
in both cases. We denote term $f_h(x_i) - \hat{f}_h(x_i)$ by $\epsilon_i$. Since we condition on \eqref{eq:bound-Q} and \eqref{eq:bound-v}, we have $$
\sum_{i=1}^n \epsilon_i^2 \leq \Ocal\bigg( {H^2 e^{2H} \cdot |\Acal|^2\cdot \log\big(H\cdot N(\Mcal_h, \|\cdot\|_{\infty}, 1/n )/\delta\big)}\bigg)
$$
for both cases (1) and (2). Meanwhile, we also have \begin{align*}
\eqref{eq:bound-target}^2 &=  \bigg(\sum_{i=1}^n\epsilon_i\phi(\thissai)\bigg)^\top\bigg(\lambda I+\sum_{i=1}^n \phi(\thissai)\phi(\thissai)^\top\bigg)^{-1}\bigg(\sum_{i=1}^n\epsilon_i\phi(\thissai)\bigg)\\
& = \operatorname{Tr}\bigg(\bigg(\sum_{i=1}^n\epsilon_i\phi(\thissai)\bigg)\bigg(\sum_{i=1}^n\epsilon_i\phi(\thissai)\bigg)^\top \bigg(\lambda I+\sum_{i=1}^n \phi(\thissai)\phi(\thissai)^\top\bigg)^{-1}\bigg).
\end{align*}
By Lemma \ref{lem:cauchy-matrix}, we have $$
\bigg(\sum_{i=1}^n\phi(\thissai)\epsilon_i\bigg)\bigg(\sumn\phi(\thissai)\epsilon_i\bigg)^\top \leq \Ocal\bigg(H^2e^{2H}\cdot|\Acal|^2\cdot\log(H\cdot N(\Theta,\|\cdot\|_\infty,1/n)/\delta)\bigg)\cdot \bigg(\sum_{i=1}^n \phi(\thissai)\phi(\thissai)^\top\bigg),
$$
and therefore we have 
\begin{align}\label{eq:2-bound}
&\operatorname{Tr}\bigg(\bigg(\sum_{i=1}^n\epsilon_i\phi(\thissai)\bigg)\bigg(\sum_{i=1}^n\epsilon_i\phi(\thissai)\bigg)^\top \bigg(\lambda I+\sum_{i=1}^n \phi(\thissai)\phi(\thissai)^\top\bigg)^{-1}\bigg)\nonumber\\
&\quad\quad \leq \Ocal\bigg(H^2e^{2H}\cdot|\Acal|\cdot\log(N(\Theta,\|\cdot\|_\infty,1/n)/\delta) \bigg)\nonumber\\
&\qquad\qquad\cdot\operatorname{Tr}\bigg(\bigg(\sumn \phi(\thissai)\phi(\thissai)^\top +\lambda\bigg)^{-1}\bigg(\sumn\phi(\thissai)\phi(\thissai)^\top\bigg)\bigg)\nonumber\\
&\quad\quad\leq  d\cdot H^2e^{2H}\cdot |\Acal|^2\log(N(\Theta,\|\cdot\|_\infty,1/n)/\delta).
\end{align}
here the last inequality comes from Lemma \ref{lem:semi-defi-positive}. Therefore we have $$
\text{(iii)}\leq \sqrt{d}He^H\cdot|\Acal|\cdot \sqrt{\log(N(\Theta,\|\cdot\|_\infty,1/n)/\delta)}
$$
Next, we bound (vi). We prove the following Lemma:
\begin{lemma}\label{lem:v-self-concen}
    Let $V:\Scal\rightarrow[0,H]$ be any fixed function. With our dataset $\Dcal = \{\Dcal_h\}_{h\in[H]}$, we have $$
    \bigg\|\sum_{i=1}^n \phi(\thissai)\bigg( \big(\P_hV_{h+1}(\thissai) - V_{h+1}(\nextstatei)\big)\bigg)\bigg\|^2_{(\Lambda_h+\lambda I)^{-1}}\leq H^2\cdot\big(2\cdot\log(H/\delta)+d\cdot\log(1+n/\lambda)\big)
    $$
    with probability at least $1-\delta$ for all $h\in[H]$.
\end{lemma}
\begin{proof}
    Note that for $V_{h+1}: \Scal\rightarrow \R$, we have $$\E[V_{h+1}(s_{h+1}^i) - \P_hV_{h+1}(s_h^i,a_h^i)\mid \Fcal_h^i] = \E[V_{h+1}(s_{h+1}^i) - \P_hV_{h+1}(s_h^i,a_h^i)\mid s_h^i,a_h^i] =0.$$
Here $\Fcal_h^i = \sigma\big(\{(s_t^i,a_t^i\}_{t=1}^h\big)$ is the filtration generated by state-action pair before step $h+1$ for the $i$-th trajectory.
We now invoke Lemma \ref{lem:self-norm-concen} with $M_0 =\lambda I$ and $M_n = \lambda I+\sum_{i=1}^n \phi(\thissai)\phi(\thissai)^\top $. We have \begin{align}\label{eq:(vi)}
    &\bigg\|\sum_{i=1}^n \phi(\thissai)\bigg( \big(\P_hV_{h+1}(\thissai) - V_{h+1}(\nextstatei)\big)\bigg)\bigg\|^2_{(\Lambda_h+\lambda I)^{-1}}\nonumber\\
    &\qquad\leq 2H^2\cdot\log\bigg(H\cdot\frac{\operatorname{det}(\Lambda_h+\lambda I)^{1/2}}{\delta\cdot\operatorname{det}(\lambda I)^{1/2}}\bigg)\nonumber
\end{align}
with probability at least $1-\delta/2$. Recall that by Assumption \ref{ass:regular}, $\|\phi(s,a)\|_2\leq 1$ and therefore we have $\operatorname{det}( \Lambda_h + \lambda I)  \leq (\lambda + n)^d$. Also we have $\operatorname{det}(\lambda I) = \lambda^d$, and we have \begin{align}
    &\bigg\|\sum_{i=1}^n \phi(\thissai)\bigg( \big(\P_hV_{h+1}(\thissai) - V_{h+1}(\nextstatei)\big)\bigg)\bigg\|^2_{(\Lambda_h+\lambda I)^{-1}}\nonumber\\
    &\qquad\leq H^2\cdot\big(2\cdot\log(H/\delta)+d\cdot\log(1+n/\lambda)\big).\nonumber
\end{align}
\end{proof}

By \eqref{ineq:bellman-transform}, \eqref{eq:2-bound}, Lemma \ref{lem:v-self-concen} and Assumption \ref{ass:regular}, we have $$
\text{(iii)}\leq \Ocal\bigg((1+\gamma)\cdot \sqrt{d\cdot H^2e^{2H}\cdot|\Acal|\cdot\log\big(H\cdot N(\Theta,\|\cdot\|_\infty,1/n)/\delta\big)}\bigg)\leq \Ocal\bigg((1+\gamma)\cdot dHe^H\sqrt{|\Acal|\log(nH/\lambda\delta)}\bigg).
$$
Therefore $\text{(ii) }\leq \Ocal\bigg((1+\gamma)\cdot |\Acal| \cdot d\cdot He^H\sqrt{\log(nH/\lambda\delta)}\bigg)\cdot \|\phi(s,a)\|_{(\Lambda_h+\lambda I)^{-1}}$. Combined with \eqref{eq:bound(i)}, we conclude the proof of Theorem \ref{thm:reward-est}.
\section{Proof for Theorem \ref{thm:pess-vi-linear-cover}}\label{sec:prove-pess-vi}
In this section, we prove Theorem \ref{thm:pess-vi-linear-cover}. First, we invoke the following theorem, whose proof can be found in \cite{jin2021pessimism}, Appendix 5.2.
\begin{theorem}[Theorem 4.2 in \cite{jin2021pessimism}]\label{thm:jin-main}
Suppose $\left\{\Gamma_h\right\}_{h=1}^H$ in Algorithm \ref{alg:pess-value-iter} is a uncertainty quantifier defined in \eqref{eq:uncertainty-quant}. Under the event which  \eqref{eq:uncertainty-quant} holds, suboptimality of Algorithm \ref{alg:pess-value-iter} satisfies
$$
\operatorname{SubOpt}(\{\tilde{\pi}_h\}_{h\in[H]}) \leq 2 \sum_{h=1}^H \mathbb{E}_{\pi^*}\left[\Gamma_h\left(s_h, a_h\right) \right] .
$$
Here $\mathbb{E}_{\pi^*}$ is with respect to the trajectory induced by $\pi^*$ in the underlying MDP.
\end{theorem}
With Theorem \eqref{thm:jin-main}, our proof for Theorem \eqref{thm:pess-vi-linear-cover} then proceeds in \textbf{two steps}: (1) We prove that our uncertainty quantifier defined in \ref{eq:penalty-linear}, with $\beta$ defined in \ref{alg:pess-value-iter}, is an uncertainty quantifier, with probability at least $1-\delta/2$; (2) We prove that with penalty function set in \ref{eq:penalty-linear}, we can bound $
\sum_{h=1}^H \mathbb{E}_{\pi^*}\left[\Gamma_h\left(s_h, a_h\right) \right]
$
with probability at least $1-\delta/2$.
\paragraph{Step (1).} We now prove that $\Gamma_h$ defined in \ref{eq:penalty-linear} is an uncertainty quantifier, with $$\beta = \Ocal\bigg({He^{H}}\cdot|\Acal|\cdot d\sqrt{\log\big({nH}/{\delta}\big)}\bigg).$$
We have \begin{align*}
    &\big|\big(\widehat{r}_h + \widetilde{\P}_h\widetilde{V}_{h+1}\big)(s,a) - \big(r_h + \P_h\widetilde{V}_{h+1}\big)(s,a)   \big| \\
    &\qquad\leq \underbrace{\big|\widehat{r}_h(s,a) - r_h(s,a)\big|}_{\text{(i)}} + \underbrace{\big|\widetilde{\P}_h\widetilde{V}_{h+1}(s,a) -   \P_h\widetilde{V}_{h+1}(s,a)\big|}_{\text{(ii)}},
\end{align*}

To bound (i), recall that we construct $\widehat{r}_h$ by Algorithm \ref{alg:reward-learn} with guarantee $$
|r_h(s,a) - \widehat{r}_h(s,a)|\leq \|\phi(s,a)\|_{(\Lambda_h+\lambda I)^{-1}} \cdot \Ocal\big((1+\gamma)\cdot He^H |\Acal|\cdot d \sqrt{\log(nH/\delta)}\big)
$$
for all $(s,a)\in\Scal\times\Acal$ with $\lambda = 1$. 
To bound (ii), recall that we construct $\widetilde{\P}_h\widetilde{V}_{h+1}(s,a) = \phi(s,a)\cdot \widetilde{u}_h $ by the Algorithm \ref{alg:pess-value-iter},  $$
\widetilde{u}_h = \operatorname{argmin}_{u}\sum_{i=1}^n \big(\phi(\thissai)\cdot u - \widetilde{V}_{h+1}(\nextstatei)\big)^2 + \lambda\cdot\|u\|^2,
$$
note that we have a closed form solution for $\widetilde{u}_h$,\begin{align*}
    \widetilde{u}_h = \big(\Lambda_h + \lambda I\big)^{-1} \bigg(\sum_{i=1}^n \phi(\thissai) \widetilde{V}_{h+1}(\nextstatei) \bigg),
\end{align*}
And by Assumption \ref{ass:linear-transition}, we have $\P_h\widetilde{V}_{h+1}(s,a) = \phi(s,a)\cdot u_h$ with $\|u_h\|\leq (H-h+1)\sqrt{d}$, therefore we have \begin{align*}
    \big|\widetilde{\P}_h\widetilde{V}_{h+1}(s,a) - \P_h \widetilde{V}_{h+1}(s,a)\bigg| & = |\phi(s,a)\big(u_h - \widetilde{u}_h\big)|\\
    & = \bigg|\phi(s,a)\big(\Lambda_h+\lambda I\big)^{-1}\bigg(\sum_{i=1}^n \phi(\thissai)\big(\widetilde{V}_{h+1}(\nextstatei) - \P_h\widetilde{V}_{h+1}(\thissai)\big)\bigg)\\
    &\qquad+ \phi(s,a)\big(\Lambda_h
    +\lambda I\big)^{-1}u_h\bigg|\\
    &\leq \bigg|\phi(s,a)\big(\Lambda_h+\lambda I\big)^{-1}\bigg(\sum_{i=1}^n \phi(\thissai)\big(\widetilde{V}_{h+1}(\nextstatei) - \P_h\widetilde{V}_{h+1}(\thissai)\big)\bigg)\bigg|\\
    &\qquad+\lambda\cdot\big|\phi(s,a)\big(\Lambda_h
    +\lambda I\big)^{-1}u_h\big|,
\end{align*}
and with Caucht-Schwarz inequality we have \begin{align}
\big|\widetilde{\P}_h\widetilde{V}_{h+1}(s,a) - \P_h \widetilde{V}_{h+1}(s,a)\big| &\leq \|\phi(s,a)\|_{\big(\Lambda_h
    +\lambda I\big)^{-1}}\cdot\bigg(
    \lambda\|u_h\|_{(\Lambda_h+\lambda I)^{-1}}\nonumber
    \\&\qquad+ \bigg\|\bigg(\sum_{i=1}^n \phi(\thissai)\big(\widetilde{V}_{h+1}(\nextstatei) - \P_h\widetilde{V}_{h+1}(\thissai)\big)\bigg)\bigg\|_{(\Lambda_h+\lambda I)^{-1}} \nonumber
    \bigg)\\
    &\leq \|\phi(s,a)\|_{(\Lambda_h+\lambda I)^{-1}}\cdot\bigg(H\sqrt{\lambda d} \nonumber\\
    &\qquad+ \underbrace{ \bigg\|\bigg(\sum_{i=1}^n \phi(\thissai)\big(\widetilde{V}_{h+1}(\nextstatei) - \P_h\widetilde{V}_{h+1}(\thissai)\big)\bigg)\bigg\|_{(\Lambda_h+\lambda I)^{-1}}}_{\text{(iii)}}
    \bigg).\nonumber
\end{align}
Completing the first step now suffices to bound (iii). However, (iii) is a self-normalizing summation term with $\widetilde{V}_{h+1}$ depends on  dataset $\{(s_t^i,a_t^i)\}_{t>h, i\in[n]}$, therefore we cannot directly use Lemma \ref{lem:self-norm-concen}. We first prove the following lemma, which bound $\|\widetilde{u}_h + \widehat{w}_h\|$.
\begin{lemma}
    In Algorithm \ref{alg:pess-value-iter}, we have $$
    \|\widetilde{u}_h + \widehat{w}_h\|\leq 2H\sqrt{nd/\lambda}.
    $$
\end{lemma}
\begin{proof}
    For the proof we only need to bound $\|\widetilde{u}_h\|$ and $\|\widehat{w}_h\|$ respectively. First we have \begin{align*}
        \|\widetilde{u}_h\|&= \bigg\|\big(\Lambda_h + \lambda I\big)^{-1} \bigg(\sum_{i=1}^n \phi(\thissai) \widetilde{V}_{h+1}(\nextstatei) \bigg)\bigg\|\\
        &\leq H\cdot \sum_{i=1}^n \big\|\big(\Lambda_h + \lambda I\big)^{-1}\phi(s_h^i,a_h^i)\big\|\\
        &\leq H\sqrt{n/\lambda}\cdot\sqrt{\operatorname{Tr}\big( \big(\Lambda_h + \lambda I\big)^{-1}\Lambda_h  \big)}\\
        &\leq H\sqrt{nd/\lambda}.
    \end{align*}
    Here the first inequality comes from $\widetilde{V}_{h+1} \in [0,H]$, and the second inequality comes from Jensen's inequality. Similarly, we have \begin{align*}
        \|\widehat{w}_h\|&= \bigg\|\big(\Lambda_h + \lambda I\big)^{-1} \bigg(\sum_{i=1}^n \phi(\thissai) \big(\widehat{Q}_h(s_h^i,a_h^i) - \widehat{V}_{h+1}(s_{h+1}^i)\big)\bigg)\bigg\|\\
        &\leq H\sqrt{nd/\lambda}.
    \end{align*}
    Therefore we complete the proof.
\end{proof}
With $\|\widetilde{u}_h + \widehat{w}_h\|$ bounded, we can now invoke Theorem \ref{thm:ullm-pessi} to bound term (iii). Set $R_0 = 2H\sqrt{nd/\lambda}$, $B = 2\beta$, $\lambda = 1$ and $\epsilon = dH/n$, we have \begin{align*}
    \text{(iii)} &\leq \sup _{V \in \mathcal{V}_{h+1}(R, B, \lambda)}\left\|\sum_{i=1}^n \phi\left(s_h^i,a_h^i\right) \cdot \big(V(s_{h+1}^i) - \E[V(s_{h+1})\mid s_h^i,a_h^i]\big)\right\|_{(\Lambda_h+\lambda I)^{-1}}\\
    &\leq \Ocal(dH\cdot\log(dHn/\delta)).
\end{align*}
Therefore we have $$
\big|\widetilde{\P}_h\widetilde{V}_{h+1}(s,a) - \P_h \widetilde{V}_{h+1}(s,a)\big| \leq \|\phi(s,a)\|_{(\Lambda_h+\lambda I)^{-1}}\cdot\Ocal\big(dH\cdot\log(dHn/\delta)+H\sqrt{d}\big).
$$
Set $\lambda = 1$ in Theorem \ref{thm:reward-est}, we have \begin{equation}\label{eq:bounded-bellman}
|r_h(s,a) - \widehat{r}_h(s,a)| + \big|\widetilde{\P}_h\widetilde{V}_{h+1}(s,a) - \P_h \widetilde{V}_{h+1}(s,a)\big|\leq \beta \|\phi(s,a)\|_{(\Lambda_h+\lambda I)^{-1}}
\end{equation}
holds with probability at least $1-\delta$.
Recall that $\Gamma_h = \beta \|\phi(s,a)\|_{(\Lambda_h+\lambda I)^{-1}}$, we prove that $\Gamma_h$ is an uncertainty quantifier defined in \ref{eq:uncertainty-quant}. To finish the proof of Theorem \ref{thm:pess-vi-linear-cover}, it suffices to finish the proof  of the second step, i.e., we bound the term $$
\sum_{i=1}^n \E_{\pi^*}[\Gamma_h(s_h,a_h)] = \sum_{i=1}^n \beta\cdot \E_{\pi^*}[\|\phi(s_h,a_h)\|_{(\Lambda_h+\lambda I)^{-1}}] .
$$
\paragraph{Step (2).}
By  Cauchy-Schwarz inequality, we have
\begin{equation}\label{eq:jensen}
\begin{aligned}
\mathbb{E}_{\pi^*} & {\left[\left(\phi\left(s_h, a_h\right)^{\top} (\Lambda_h+\lambda I)^{-1} \phi\left(s_h, a_h\right)\right)^{1 / 2} \right] } \\
& =\mathbb{E}_{\pi^*}\left[\sqrt{\operatorname{Tr}\left(\phi\left(s_h, a_h\right)^{\top} (\Lambda_h+\lambda I)^{-1} \phi\left(s_h, a_h\right)\right) } \right] \\
& =\mathbb{E}_{\pi^*}\left[\sqrt{\operatorname{Tr}\left(\phi\left(s_h, a_h\right) \phi\left(s_h, a_h\right)^{\top} (\Lambda_h+\lambda I)^{-1}\right) } \right] \\
& \leq \sqrt{\operatorname{Tr}\left(\mathbb{E}_{\pi^*}\left[\phi\left(s_h, a_h\right) \phi\left(s_h, a_h\right)^{\top} \right] \Lambda_h^{-1}\right)}
\end{aligned}
\end{equation}

for all $h\in[H]$.
For notational simplicity, we define
$$
\Sigma_h=\mathbb{E}_{\pi^*}\big[\phi\left(s_h, a_h\right) \phi\left(s_h, a_h\right)^{\top} \big]
$$
for all $h\in[H]$. Condition on the event in Equation \eqref{eq:bounded-bellman} and with Assumption \ref{ass:coverage-ass}, we have 
\begin{align*}
       \operatorname{SubOpt}\big(\{\widetilde{\pi}_h\}_{h\in[H]}\big) &\leq 2\beta\cdot \sum_{h=1}^H \E_{\pi^*}\bigg[\phi(s_h,a_h)^\top(\Lambda_h+\lambda\cdot I)^{-1}\phi(s_h,a_h) \bigg]\\
       &\begin{aligned}
& \leq 2 \beta \sum_{h=1}^H \sqrt{\operatorname{Tr}\left(\Sigma_h \cdot\left(I+c^{\dagger} \cdot n \cdot \Sigma_h\right)^{-1}\right)} \\
& =2 \beta \sum_{h=1}^H \sqrt{\sum_{j=1}^d \frac{\lambda_{h, j}}{1+c^{\dagger} \cdot n \cdot \lambda_{h, j}}},
\end{aligned}
\end{align*}
here $\{\lambda_{h,j}\}_{j=1}^d$ are the eigenvalues of $\Sigma_h$. The first inequality comes from the event in Equation \eqref{eq:bounded-bellman}, and the second inequality comes from Equation \eqref{eq:jensen}. Meanwhile, by Assumption \ref{ass:regular}, we have $\|\phi(s,a)\|\leq 1$ for all $(s,a)\in\Scal\times\Acal$. 
By Jensen’s inequality, we have $$
\|\Sigma_h\|_{2} \leq \mathbb{E}_{\pi^*}\big[\|\phi\left(s_h, a_h\right) \phi\left(s_h, a_h\right)^{\top}\|_{2}\big] \leq 1
$$ for all $h\in[H]$, for all $s_h \in \mathcal{S}$ and all $h \in[H]$. As $\Sigma_h$ is positive semidefinite, we have $\lambda_{h, j} \in[0,1]$ for all $x \in \mathcal{S}$, all $h \in[H]$, and all $j \in[d]$. Hence, we have
$$
\begin{aligned}
\operatorname{SubOpt}(\{\widetilde{\pi}_h\}_{h\in[H]}) & \leq 2 \beta \sum_{h=1}^H \sqrt{\sum_{j=1}^d \frac{\lambda_{h, j}}{1+c^{\dagger} \cdot n \cdot \lambda_{h, j}}} \\
& \leq 2 \beta \sum_{h=1}^H \sqrt{\sum_{j=1}^d \frac{1}{1+c^{\dagger} \cdot n}} \leq c^{\prime} \cdot d^{3 / 2} H^2 n^{-1 / 2} \sqrt{\xi},
\end{aligned}
$$
where $\xi = \sqrt{\log(dHn/\delta)}$, the second inequality follows from the fact that $\lambda_{h, j} \in[0,1]$ for all  $h \in[H]$, and all $j \in[d]$, while the third inequality follows from the choice of the scaling parameter $\beta>0$ in Theorem \ref{thm:pess-vi-linear-cover}. Here we define the absolute constant $c^{\prime}=2 c / \sqrt{c^{\dagger}}>0$, where $c^{\dagger}>0$ is the absolute constant used in Assumption \ref{ass:coverage-ass}. Moreover, we consider the case of $\operatorname{rank}(\Sigma_h)\leq r$. Then we have 
\begin{align*}
    \operatorname{SubOpt}(\{\widetilde{\pi}_h\}_{h\in[H]}) & \leq 2 \beta \sum_{h=1}^H \sqrt{\sum_{j=1}^d \frac{\lambda_{h, j}}{1+c^{\dagger} \cdot n \cdot \lambda_{h, j}}} \\
    &=  2 \beta \sum_{h=1}^H \sqrt{\sum_{j=1}^r \frac{\lambda_{h, j}}{1+c^{\dagger} \cdot n\cdot\lambda_{h, j}}}  \\
& \leq 2 \beta \sum_{h=1}^H \sqrt{\sum_{j=1}^r \frac{1}{1+c^{\dagger} \cdot n}} \leq c^{\prime} \cdot r^{1/2} d H^2 n^{-1 / 2} \sqrt{\xi}
\end{align*}
Thus we finish the proof for Theorem \ref{thm:pess-vi-linear-cover}.
\section{Proof for RKHS Case}\label{sec:rkhs-proof}
In this section, we prove the results of DCPPO in RKHS model class. In the following, we adopt an equivalent set of notations for ease of presentation. We formally write the inner product in $\mathcal{H}$ as $\left\langle f, f^{\prime}\right\rangle_{\mathcal{H}}=f^{\top} f^{\prime}=f^{\prime \top} f$ for any $f, f^{\prime} \in \mathcal{H}$, so that $f(z)=\langle\phi(z), f\rangle_{\mathcal{H}}=f^{\top} \phi(z)$ for any $f \in \mathcal{H}$ and any $z \in \mathcal{Z}$. Moreover we denote the operators $\Phi_h: \mathcal{H} \rightarrow \mathbb{R}^n$ and $\Lambda_h: \mathcal{H} \rightarrow \mathcal{H}$ as
$$
\Phi_h=\left(\begin{array}{c}
\phi\left(z_h^{1}\right)^{\top} \\
\vdots \\
\phi\left(z_h^{n}\right)^{\top}
\end{array}\right), \quad \Lambda_h=\lambda \cdot I_{\mathcal{H}}+\sum_{i\in[n]} \phi\left(z_h^\tau\right) \phi\left(z_h^\tau\right)^{\top}=\lambda \cdot I_{\mathcal{H}}+\Phi_h^{\top} \Phi_h
$$
where $I_{\mathcal{H}}$ is the identity mapping in $\mathcal{H}$ and all the formal matrix multiplications follow the same rules as those for real-valued matrix. In this way, these operators are well-defined. Also, $\Lambda_h$ is a self-adjoint operator eigenvalue no smaller than $\lambda$, in the sense that $\langle f, \Lambda_h g\rangle=\langle\Lambda_h f, g\rangle$ for any $f, g \in \mathcal{H}$. Therefore, there exists a positive definite operator $\Lambda_h^{1 / 2}$ whose eigenvalues are no smaller than $\lambda^{1 / 2}$ and $\Lambda_h=\Lambda_h^{1 / 2} \Lambda_h^{1 / 2}$. We denote the inverse of $\Lambda_h^{1 / 2}$ as $\Lambda_h^{-1 / 2}$, so that $\Lambda_h^{-1}=\Lambda_h^{-1 / 2} \Lambda_h^{-1 / 2}$ and $\left\|\Lambda_h^{-1 / 2}\right\|_{\mathcal{H}} \leq \lambda^{-1 / 2}$. For any $z \in \mathcal{Z}$, we denote $\Lambda_h(z)=\Lambda_h+\phi(z) \phi(z)^{\top}$. 
In particular, it holds that $$
\left(\Phi_h^{\top} \Phi_h+\lambda \cdot I_{\mathcal{H}}\right) \Phi_h^{\top}=\Phi_h^{\top}\left(\Phi_h \Phi_h^{\top}+\lambda \cdot I\right).
$$
Since both the matrix $\Phi_h \Phi_h^{\top}+\lambda \cdot I$ and the operator $\Phi_h^{\top} \Phi_h+\lambda \cdot I_{\mathcal{H}}$ are strictly positive definite, we have
\begin{equation}\label{eq:transform}
\Phi_h^{\top}\left(\Phi_h \Phi_h^{\top}+\lambda \cdot I\right)^{-1}=\left(\Phi_h^{\top} \Phi_h+\lambda \cdot I_{\mathcal{H}}\right)^{-1} \Phi_h^{\top}.
\end{equation}
Our learning process would depend on the "complexity" of the dataset sampled by $\pi_b$. To measure this complexity, we make the following definition.\begin{definition}[Effective dimension]\label{def:eff-dim}
     For all $h \in[H]$, Denote $\Sigma_h=\mathbb{E}_{\pi^b}\big[\phi\left(z_h\right) \phi\left(z_h\right)^{\top} \big], \Sigma_h^*=\mathbb{E}_{\pi^*}\big[\phi\left(z_h\right) \phi\left(z_h\right)^{\top} \big]$, where $\mathbb{E}_{\pi^*}$ is taken with respect to $\left(s_h, a_h\right)$ induced by the optimal policy $\pi^*$, and $\mathbb{E}_{\pi^b}$ is similarly induced by the behavior policy $\pi^b$. We define the (sample) effective dimension as
$$
d_{\text {eff }}^{\text {sample }}=\sum_{h=1}^H \operatorname{Tr}\left(\left(\Lambda_h+\lambda \mathcal{I}_{\mathcal{H}}\right)^{-1} \Sigma_h\right)^{1 / 2} .
$$
Moreover, we define the population effective dimension under $\pi^b$ as
$$
d_{\text {eff }}^{p o p}=\sum_{h=1}^H \operatorname{Tr}\left(\left(n \cdot \Sigma_h+\lambda \mathcal{I}_{\mathcal{H}}\right)^{-1} \Sigma_h^*\right)^{1 / 2}.
$$
\end{definition}
\subsection{Proof for Theorem \ref{thm:reward-est-rkhs}}\label{sec:reward-est-rkhs}
Our proof for reward estimation in RKHS model class is very similar to the proof of linear model MDP, which can be found in Section \ref{sec:prove-reward-est}.
To prove Theorem \ref{thm:reward-est-rkhs}, we first invoke Theorem \ref{thm:emp-mle-guarantee}, and we have $$
\E_{\Dcal_h}\big[\|\learnedpolicy(\cdot\mid\thisstate) - \behpolicy(\cdot\mid \thisstate)\|_1^2\big] \leq \Ocal\bigg( \frac{\log\big(H\cdot N(\Qcal, \|\cdot\|_{\infty}, 1/n )/\delta\big)}{n}\bigg)
    $$ and 
    \begin{equation}\label{eq:bound-Q-rkhs}
\E_{\Dcal_h}\big[\|\learnedq(\thisstate,\cdot) - \behaviorq(\thisstate, \cdot)\|_1^2\big] \leq \Ocal\bigg( \frac{H^2 e^{2H} \cdot |\Acal|^2\cdot \log\big(H\cdot N(\Qcal, \|\cdot\|_{\infty}, 1/n )/\delta\big)}{n}\bigg)
\end{equation}
hold for every $h\in[H]$ with probability at least $1-\delta/2$. Here the model class $\Qcal$ is defined in Assumption \ref{ass:rkhs-retain}. Conditioning on this event, we have \begin{equation}\label{eq:bound-v-rkhs}
\E_{\Dcal_h}\big[|\widehat{V}_{h+1}(s_{h+1}) - V^{\pi_b,\gamma}_{h+1}(s_{h+1})|^2\big]\leq \Ocal\bigg( \frac{H^2 e^{2H} \cdot |\Acal|^2\cdot \log\big(H\cdot N(\Qcal, \|\cdot\|_{\infty}, 1/n )/\delta\big)}{n}\bigg)
\end{equation}
for all $h\in[H]$ 
and all $s_{h+1}\in\Scal$ simultaneously.
By Algorithm \ref{alg:reward-learn}, we have $$
    \widehat{r}_h = (\Lambda_h + \lambda I)^{-1} \bigg( \sum_{i=1}^n \phi(z_h^i)\big(\learnedq(z_h^i) - \gamma\cdot\learnednextv(\nextstatei)\big)\bigg),
$$

Recall that we denote $(s,a)\in\Scal\times\Acal$ by $z\in\Zcal$. Since we have $r_h(z) = \phi(z)\cdot r_h$. By our construction for $\widehat{r}_h$ in Algorithm \ref{alg:reward-learn}, we therefore have \begin{align*}
    |r_h(z) - \widehat{r}_h(z)| &= |\phi(z) (r_h - \widehat{r}_h)|\\
    &= \bigg|\phi(z)\big(\Lambda_h + \lambda I\big)^{-1}\bigg(\lambda\cdot r_h+\sum_{i=1}^n \phi(z_h^i)\big(\learnedq(z_h^i) - \gamma\cdot \widehat{V}_{h+1}(s_{h+1}^i) - r_h(z_h^i)\big)\bigg)\bigg|\\
    &\leq \underbrace{\lambda\cdot|\phi(z)\big(\Lambda_h+\lambda I\big)^{-1}r_h|}_{\text{(i)}} \\
    &\qquad+ \underbrace{\bigg|\phi(z)\big(\Lambda_h+\lambda I\big)^{-1}\bigg(\sum_{i=1}^n \big(\learnedq(z_h^i) - \gamma\cdot \widehat{V}_h(s_{h+1}^i) - r_h(z_h^i)\big)\bigg)\bigg|}_{\text{(ii)}}
\end{align*}
holds for all $z\in\Zcal$.
For (i), we have $$
\text{(i)}\leq \lambda\cdot\|\phi(z)\|_{(\Lambda_h + \lambda I)^{-1}}\cdot \|r_h\|_{(\Lambda_h + \lambda I)^{-1}},
$$
by Cauchy-Schwarz inequality and by $\Lambda_h$ being semi-positive definite and $\|r_h\|_\Hcal\leq R_r$, we have \begin{equation}\label{eq:bound(i)-rkhs}
\text{(i)}\leq \lambda \cdot R_r/\sqrt{\lambda}\cdot\|\phi(z)\|_{(\Lambda_h + \lambda I)^{-1}} = \sqrt{\lambda }\cdot R_r\cdot \|\phi(z)\|_{(\Lambda_h + \lambda I)^{-1}},
\end{equation}

and $$
\text{(ii)} \leq \|\phi(z)\|_{(\Lambda_h + \lambda I)^{-1}} \underbrace{\cdot \bigg\|\sum_{i=1}^n \phi(z_h^i)\big(\learnedq(z_h^i) - \gamma\cdot \widehat{V}_{h+1}(\nextstatei) - r_h(z_h^i)\big)\bigg\|_{(\Lambda_h+\lambda I)^{-1}}}_{\text{(iii)}} .
$$
Recall that in \eqref{eq:discounted-bellman}, we have the following Bellman equation hold for all $(\thissa)\in\Scal\times\Acal$,
$$
r_h(z_h) + \gamma\cdot \P_h V_{h+1}^{\pi_b,\gamma}(z_h) = \behaviorq(z_h),
$$
substitute this into (iii), and we have  \begin{align}\label{ineq:bellman-transform}
\text{(iii)} &= \bigg\|\sum_{i=1}^n \phi(z_h^i)\bigg(\big(\learnedq(z_h^i)-\behaviorq(z_h^i)\big) - \gamma\cdot \big(\widehat{V}_{h+1}(\nextstatei) - \P_hV^{\pi_b,\gamma}_{h+1}(z_h^i)\big)\bigg)\bigg\|_{(\Lambda_h+\lambda I)^{-1}}\nonumber\\
&\leq \underbrace{\bigg\|\sum_{i=1}^n \phi(z_h^i)\bigg(\big(\learnedq(z_h^i)-\behaviorq(z_h^i)\big) \bigg)\bigg\|_{(\Lambda_h+\lambda I)^{-1}}}_{\text{(iv)}}\nonumber\\
&\qquad+ \gamma\cdot\underbrace{\bigg\|\sum_{i=1}^n \phi(z_h^i)\bigg( \big(\widehat{V}_{h+1}(\nextstatei) - V^{\pi_b,\gamma}_{h+1}(\nextstatei)\big)\bigg)\bigg\|_{(\Lambda_h+\lambda I)^{-1}}}_{\text{(v)}}\nonumber\\
&\qquad+\gamma\cdot\underbrace{\bigg\|\sum_{i=1}^n \phi(z_h^i)\bigg( \big(\P_hV^{\pi_b,\gamma}_{h+1}(z_h^i) - V^{\pi_b,\gamma}_{h+1}(\nextstatei)\big)\bigg)\bigg\|_{(\Lambda_h+\lambda I)^{-1}}}_{\text{(vi)}},
\end{align}
First, we bound (iv) and (v). By Theorem \ref{thm:emp-mle-guarantee}, we have $$
\E_{\Dcal_h}\big[\|\learnedpolicy(\cdot\mid\thisstate) - \behpolicy(\cdot\mid \thisstate)\|_1^2\big] \leq \Ocal\bigg( \frac{\log\big(H\cdot N(\Qcal, \|\cdot\|_{\infty}, 1/n )/\delta\big)}{n}\bigg)
    $$ and 
    \begin{equation}\label{eq:bound-Q-rkhs}
\E_{\Dcal_h}\big[\|\learnedq(\thisstate,\cdot) - \behaviorq(\thisstate, \cdot)\|_1^2\big] \leq \Ocal\bigg( \frac{H^2 e^{2H} \cdot |\Acal|^2\cdot \log\big(H\cdot N(\Qcal, \|\cdot\|_{\infty}, 1/n )/\delta\big)}{n}\bigg)
\end{equation}
hold for every $h\in[H]$ with probability at least $1-\delta/2$. By  $\widehat{V}_h(s) = \langle \widehat{Q}_h(s,\cdot), \learnedpolicy(\cdot\mid s)\rangle_{\Acal}$ for every $s_{h+1}\in\Scal$  , we have \begin{equation}\label{eq:bound-v-rkhs}
\E_{\Dcal_h}\big[|\widehat{V}_{h+1}(s_{h+1}) - V^{\pi_b,\gamma}_{h+1}(s_{h+1})|^2\big]\leq \Ocal\bigg( \frac{H^2 e^{2H} \cdot |\Acal|^2\cdot \log\big(H\cdot N(\Qcal, \|\cdot\|_{\infty}, 1/n )/\delta\big)}{n}\bigg)
\end{equation}
for all $h\in[H]$ simultaneously.
In the following proof, we will condition on these events.
For notation simplicity, we define two function $f:\Xcal\rightarrow \R$ and $\widehat{f}:\Xcal\rightarrow \R$ for each $h\in[H]$, and dataset $\{x_i\}_{i\in[n]}$. We consider two cases:  (1)  $\widehat{f}_h = \learnedq$, $f_h = \behaviorq$, and  $x_i = z_h^i$, $ \Xcal = \Zcal,$  (2) $\widehat{f}_h = \learnednextv$, $f_h = V_{h+1}^{\pi_b,\gamma}$, and $x_i= \nextstatei$ , $\Xcal = \Scal$.  To bound (iv) and (v), we only need to uniformly bound \begin{equation}\label{eq:bound-target-rkhs}
\bigg\|\sum_{i=1}^n \phi(z_h^i) \big(f_h(x_i) - \hat{f}_h(x_i)\big)\bigg\|_{(\Lambda_h+\lambda I)^{-1}}.
\end{equation}
in both cases. We denote term $f_h(x_i) - \hat{f}_h(x_i)$ by $\epsilon_i$. Recall that we condition on \eqref{eq:bound-Q} and \eqref{eq:bound-v}, we have $$
\sum_{i=1}^n \epsilon_i^2 \leq \Ocal\bigg( \frac{H^2 e^{2H} \cdot |\Acal|^2\cdot \log\big(H\cdot N(\Qcal, \|\cdot\|_{\infty}, 1/n )/\delta\big)}{n}\bigg)
$$
for both cases (1) and (2). Meanwhile, we also have \begin{align*}
\eqref{eq:bound-target}^2 &=  \bigg(\sum_{i=1}^n\epsilon_i\phi(z_h^i)\bigg)^\top\bigg(\lambda I+\sum_{i=1}^n \phi(z_h^i)\phi(z_h^i)^\top\bigg)^{-1}\bigg(\sum_{i=1}^n\epsilon_i\phi(z_h^i)\bigg)\\
& = \operatorname{Tr}\bigg(\bigg(\sum_{i=1}^n\epsilon_i\phi(z_h^i)\bigg)\bigg(\sum_{i=1}^n\epsilon_i\phi(z_h^i)\bigg)^\top \bigg(\lambda I_\Hcal+\Phi_h^\top\Phi_h\bigg)^{-1}\bigg).
\end{align*}
By Lemma \ref{lem:cauchy-matrix}, we have $$
\bigg(\sum_{i=1}^n\phi(z_h^i)\epsilon_i\bigg)\bigg(\sumn\phi(z_h^i)\epsilon_i\bigg)^\top \leq \Ocal\bigg(H^2e^{2H}\cdot|\Acal|^2\cdot\log(H\cdot N(\Qcal,\|\cdot\|_\infty,1/n)/\delta)\bigg)\cdot \bigg(\sum_{i=1}^n \Phi_h^\top\Phi_h\bigg),
$$

For notation simplicity, denote $\phi(z_h)$ by $\bm{u}_i$,
\begin{align}\label{eq:2-bound}
&\operatorname{Tr}\bigg(\bigg(\sum_{i=1}^n\epsilon_i\phi(z_h)\bigg)\bigg(\sum_{i=1}^n\epsilon_i\phi(z_h)\bigg)^\top \bigg(\lambda I_\Hcal+ \Phi_h^\top\Phi_h\bigg)^{-1}\bigg)\nonumber\\
&\quad\quad \leq \Ocal\bigg(\bigg(H^2e^{2H}\cdot|\Acal|\cdot\log(N(\Qcal,\|\cdot\|_\infty,1/n)/\delta) \bigg)\nonumber\\
&\qquad\qquad\cdot\operatorname{Tr}\bigg(\bigg(\Phi_h^\top\Phi_h+\lambda\Ical_\Hcal\bigg)^{-1}\bigg(\Phi_h^\top\Phi_h\bigg)\bigg)\nonumber\\
&\quad\quad\leq  \dimsamp^2\cdot H^2e^{2H}\cdot |\Acal|^2\log(N(\Qcal,\|\cdot\|_\infty,1/n)/\delta).
\end{align}
here the last inequality comes from the definition of $\dimsamp$ and Lemma D.3 in \cite{jin2021pessimism}. Since there is no distribution shift, the effective dimension can be bounded by a constant.
Next, we bound (vi). We prove the following lemma, which is the RKHS version of Lemma \ref{lem:v-self-concen}.
\begin{lemma}\label{lem:v-self-concen-rkhs}
    Let $V:\Scal\rightarrow[0,H]$ be any fixed function. With our dataset $\Dcal = \{\Dcal_h\}_{h\in[H]}$, we have $$
    \bigg\|\sum_{i=1}^n \phi(z_h^i)\bigg( \big(\P_hV_{h+1}(z_h^i) - V_{h+1}(\nextstatei)\big)\bigg)\bigg\|^2_{(\Lambda_h+\lambda I)^{-1}}\leq H^2\cdot G(n,1+1/n) + 2H^2 \cdot \log(H/\delta)
    $$
    with probability at least $1-\delta$ for all $h\in[H]$ when $ 1+1/n \leq \lambda$.
\end{lemma}
\begin{proof}
    Note that for $V_{h+1}: \Scal\rightarrow \R$, we have $$\E[V_{h+1}(s_{h+1}^i) - \P_hV_{h+1}(s_h^i,a_h^i)\mid \Fcal_h^i] = \E[V_{h+1}(s_{h+1}^i) - \P_hV_{h+1}(s_h^i,a_h^i)\mid s_h^i,a_h^i] =0.$$
Here $\Fcal_h^i = \sigma\big(\{(s_t^i,a_t^i\}_{t=1}^h\big)$ is the filtration generated by state-action pair before step $h+1$.
We now invoke Lemma \ref{lem:self-norm-concen-rkhs} with $\epsilon_h^i =V_{h+1}(s_{h+1}^i) - \P_hV_{h+1}(s_h^i,a_h^i)$ and $\sigma^2 = H^2$ since $\epsilon_h^i \in [-H,H]$ and it holds with probability at least $1-\delta$ for all $h\in[H]$ that \begin{align}
& E_h^{\top}\left[\left(K_h+\eta \cdot I\right)^{-1}+I\right]^{-1} E_h \\
&\qquad \leq H^2 \cdot \log \operatorname{det}\left[(1+\eta) \cdot I+K_h\right]+2 H^2 \cdot \log (H / \delta)
\end{align}
for any $\eta >0$, where $E_h = (\epsilon_h^i)^\top_{i\in[n]}$. We now transform  into the desired form, 

\begin{align}
    &\bigg\|\sum_{i=1}^n \phi(\thissai)\bigg( \big(\P_hV_{h+1}(\thissai) - V_{h+1}(\nextstatei)\big)\bigg)\bigg\|^2_{(\Lambda_h+\lambda I)^{-1}}\nonumber\\
& \qquad=E_h^{\top} \Phi_h\left(\Phi_h^{\top} \Phi_h+\lambda \cdot I_{\mathcal{H}}\right)^{-1} \Phi_h^{\top} E_h\nonumber \\
&\qquad =E_h^{\top} \Phi_h \Phi_h^{\top}\left(\Phi_h \Phi_h^{\top}+\lambda \cdot I\right)^{-1} E_h \nonumber\\
&\qquad =E_h^{\top} K_h\left(K_h+\lambda \cdot I\right)^{-1} E_h \nonumber\\
&\qquad =E_h^{\top} E_h-\lambda \cdot E_h^{\top}\left(K_h+\lambda \cdot I\right)^{-1} E_h \nonumber\\
&\qquad =E_h^{\top} E_h-E_h^{\top}\left(K_h / \lambda+I\right)^{-1} E_h,
\end{align}
where the first equality follows from the definition of $\Lambda_h$, the second equality from \eqref{eq:transform}, and the third equality follows from the fact that $K_h = \Phi_h^\top\Phi_h$ . Therefore for any $\underline{\lambda}>1$ such that $\lambda \geq \underline{\lambda}$, it holds that
$$
\bigg\|\sum_{i=1}^n \phi(\thissai)\bigg( \big(\P_hV_{h+1}(\thissai) - V_{h+1}(\nextstatei)\big)\bigg)\bigg\|^2_{(\Lambda_h+\lambda I)^{-1}} \leq E_h^{\top} K_h\left(K_h+\underline{\lambda} \cdot I\right)^{-1} E_h .
$$
For any $\eta>0$, noting that $\big(\left(K_h+\eta \cdot I\right)^{-1}+I\big)\left(K_h+\eta \cdot I\right)=K_h+(1+\eta) \cdot I$, we have
\begin{equation}\label{eq:identity}
\big(\left(K_h+\eta \cdot I\right)^{-1}+I\big)^{-1}=\left(K_h+\eta \cdot I\right)\left(K_h+(1+\eta) \cdot I\right)^{-1}
\end{equation}

Meanwhile, taking $\eta=\underline{\lambda}-1>0$, we have
$$
\begin{aligned}
E_h^{\top} K_h\left(K_h+\underline{\lambda} \cdot I\right)^{-1} E_h & \leq E_h^{\top}\left(K_h+\eta \cdot I\right)\left(K_h+\underline{\lambda} \cdot I\right)^{-1} E_h \\
& =E_h^{\top}\left[\left(K_h+\eta \cdot I\right)^{-1}+I\right]^{-1} E_h,
\end{aligned}
$$
where the second line follows from \eqref{eq:identity}. For any fixed $\delta>0$, now we know that
\begin{align}\label{ineq:bound-concen}
\bigg\|\sum_{i=1}^n \phi(\thissai)\bigg( \big(\P_hV_{h+1}(\thissai) - V_{h+1}(\nextstatei)\big)\bigg)\bigg\|^2_{(\Lambda_h+\lambda I)^{-1}} &\leq H^2 \cdot \log \operatorname{det}\left[\underline{\lambda} \cdot I+K_h\right]+2 H^2 \cdot \log (H / \delta)\nonumber\\
&\leq H^2\cdot G(n,\underline{\lambda}) + 2H^2 \cdot \log(H/\delta)
\end{align}
for all $h\in[H]$ with probability at least $1-\delta$.
\end{proof}

Combining \eqref{ineq:bellman-transform}, \eqref{eq:2-bound}, Lemma \ref{lem:v-self-concen-rkhs} and Assumption \ref{ass:regular}, and set $\underline{\lambda} = 1+1/n \leq \lambda$, we have $$
\text{(iii)}\leq \Ocal\bigg((1+\gamma)\cdot \dimsamp\cdot He^{H}\cdot |\Acal|\sqrt{\log(N(\Qcal,\|\cdot\|_\infty,1/n)/\delta)}+ \sqrt{H^2\cdot G(n,1+1/n) + 2H^2 \cdot \log(H/\delta)} \bigg),
$$
Since $\text{(ii) }\leq \text{(iii)}\cdot \|\phi(z)\|_{\Lambda_h+\lambda\Ical_{\Hcal}}$, combined with the bound for (i) in \eqref{eq:bound(i)-rkhs}, we conclude the proof of Theorem \ref{thm:reward-est-rkhs}.
\subsection{Proof for Theorem \ref{thm:pessim-vi-rkhs}}\label{sec:prove-pessi-vi-rkhs}
To prove Theorem, we again invoke Theorem \ref{thm:jin-main}. Our proof proceeds in two steps: (1) We prove that with $\beta$ set in Theorem \ref{thm:pessim-vi-rkhs}, \eqref{eq:penalty-rkhs} is an uncertainty quantifier with high probability for every $h\in[H]$. (2) We prove that with penalty function set in \eqref{eq:penalty-rkhs}, we can bound $\sum_{i=1}^n \E_{\pi^*}[\Gamma_h(z_h)]$.
\paragraph{Step (1).} In this step we prove that with $\beta$ specified in Theorem \ref{thm:pessim-vi-rkhs}, the penality functions $\{\Gamma_h\}_{h\in[H]}$ are uncertainty quantifiers with high probability. By Algorithm \ref{alg:pess-value-iter}, We have \begin{align*}
    &\big|\big(\widehat{r}_h + \widetilde{\P}_h\widetilde{V}_{h+1}\big)(s,a) - \big(r_h + \P_h\widetilde{V}_{h+1}\big)(s,a)   \big| \\
    &\qquad\leq \underbrace{\big|\widehat{r}_h(s,a) - r_h(s,a)\big|}_{\text{(i)}} + \underbrace{\big|\widetilde{\P}_h\widetilde{V}_{h+1}(s,a) -   \P_h\widetilde{V}_{h+1}(s,a)\big|}_{\text{(ii)}},
\end{align*}

To bound (i), recall that we construct $\widehat{r}_h$ by Algorithm \ref{alg:reward-learn} with guarantee \begin{align}\label{eq:bound-r-rkhs}
|r_h(s,a) - \widehat{r}_h(s,a)|&\leq\|\phi(s,a)\|_{(\Lambda_h+\lambda I)^{-1}}\nonumber\\ &\qquad\cdot \Ocal\bigg((1+\gamma) \dimsamp He^{H} |\Acal|\sqrt{\log(H\cdot N(\Qcal,\|\cdot\|_\infty,1/n)/\delta)+ G(n,1+1/n)}+ \lambda\cdot R_r^2\bigg)
\end{align}
for all $(s,a)\in\Scal\times\Acal$, $h\in[H]$ with probability at least $1-\delta/2$.  
To bound (ii), recall that we construct $\widetilde{\P}_h\widetilde{V}_{h+1}(s,a) = \phi(s,a)\cdot \widetilde{u}_h $ by the Algorithm \ref{alg:pess-value-iter},  $$
\widetilde{f}_h = \operatorname{argmin}_{f\in\Hcal}\sum_{i=1}^n \big(\phi(z_h^i)\cdot f - \widetilde{V}_{h+1}(\nextstatei)\big)^2 + \lambda\cdot\|f\|_\Hcal^2,
$$
note that by the Representer's theorem (see \citep{Steinwart2008SupportVM}), we have a closed form solution for $\widetilde{f}_h$,\begin{align*}
    \widetilde{f}_h = \big(\Phi_h^\top\Phi_h + \lambda \Ical_\Hcal\big)^{-1} \Phi_h^\top \widetilde{y}_h ,
\end{align*}
here we use the notation $\widetilde{y}_h = (\widetilde{V}_{h+1}(s_{h+1}^i))^\top$.  Meanwhile, with Assumption \ref{ass:rkhs-retain}, we have $\P_h\widetilde{V}_{h+1}(s,a) = \phi(s,a)\cdot f_h$ with $\|f_h\|_{\Hcal}\leq R_\Qcal H$, therefore we have \begin{align*}
    \big|\widetilde{\P}_h\widetilde{V}_{h+1}(s,a) - \P_h \widetilde{V}_{h+1}(s,a)\big|  = &|\phi(s,a)\big(f_h - \widetilde{f}_h\big)|\\
     = & \left|\phi(s, a)^{\top}\left(\Phi_h^{\top} \Phi_h+\lambda \cdot I_{\mathcal{H}}\right)^{-1} \Phi_h^{\top} \widetilde{y}_h-\phi(s, a)^{\top} f_h\right| \\
= & |\underbrace{\phi(s, a)^{\top}\left(\Phi_h^{\top} \Phi_h+\lambda \cdot I_{\mathcal{H}}\right)^{-1} \Phi_h^{\top} \Phi_h f_h-\phi(s, a)^{\top} f_h}_{\text {(i) }}| \\
& +|\underbrace{\phi(s, a)^{\top}\left(\Phi_h^{\top} \Phi_h+\lambda \cdot I_{\mathcal{H}}\right)^{-1} \Phi_h^{\top}\left(\widetilde{y}_h-\Phi_h f_h\right)}_{\text {(ii) }}| .
\end{align*}

In the sequel, we bound terms (i) and (ii) separately. By the Cauchy-Schwarz inequality,
$$
\begin{aligned}
|(\mathrm{i})| & =\left|\phi(s, a)^{\top}\left(\Phi_h^{\top} \Phi_h+\lambda \cdot I_{\mathcal{H}}\right)^{-1} \Phi_h^{\top} \Phi_h f_h-\phi(s, a)^{\top} f_h\right| \\
& =\left|\phi(s, a)^{\top}\left(\Phi_h^{\top} \Phi_h+\lambda \cdot I_{\mathcal{H}}\right)^{-1}\left[\Phi_h^{\top} \Phi_h-\left(\Phi_h^{\top} \Phi_h+\lambda \cdot I_{\mathcal{H}}\right)\right] f_h\right| \\
& =\lambda \cdot\left|\phi(s, a)^{\top} (\Lambda_h+\lambda\Ical_\Hcal)^{-1} f_h\right| \\
& \leq \lambda \cdot\left\|(\Lambda_h+\lambda\Ical_\Hcal)^{-1} \phi(x, a)\right\|_{\mathcal{H}} \cdot\left\|f_h\right\|_{\mathcal{H}},
\end{aligned}
$$
recall that we define $\Lambda_h=\Phi_h^{\top} \Phi_h$. Therefore, it holds that
$$
\begin{aligned}
|(\mathrm{i})|  & \leq \lambda^{1 / 2} \cdot\left\|\Lambda_h^{-1 / 2} \phi(s, a)\right\|_{\mathcal{H}} \cdot\left\|f_h\right\|_{\mathcal{H}} \\
& \leq R_r H \cdot \lambda^{1 / 2} \cdot\|\phi(s, a)\|_{(\Lambda_h+\lambda\Ical_\Hcal)^{-1}}.
\end{aligned}
$$ Here the first inequality comes from $\Lambda_h + \lambda\Ical_\Hcal \succeq \lambda \Ical_\Hcal$, and the second inequality comes from Assumption \ref{ass:rkhs-retain},
On the other hand, we have \begin{align}
    \text{(ii)} &= \big|\phi(s,a)^\top\big(\Phi_h^\top\Phi_h+\lambda\cdot \Ical_\Hcal \big)^{-1}\Phi_h^\top(\widetilde{y}_h - \Phi_hf_h)\big|\nonumber \\
    &= \bigg| \phi(s,a)^\top \big(\Phi_h^\top\Phi_h + \lambda\cdot\Ical_\Hcal\big)^{-1}\bigg(\sum_{i=1}^n \phi(s_h^i,a_h^i)\big(\widetilde{V}_{h+1}(s_{h+1}^i)- \P_h\widetilde{V}_{h+1}(s_h^i,a_h^i) \big)\bigg)\bigg|\nonumber\\
    &\leq \|\phi(s,a)\|_{(\Lambda_h+\lambda \Ical_\Hcal)^{-1}}\cdot \underbrace{\bigg\|\sum_{i=1}^n \phi(s_h^i,a_h^i)\big(\widetilde{V}_{h+1}(s_{h+1}^i)- \P_h\widetilde{V}_{h+1}(s_h^i,a_h^i) \big) \bigg\|_{(\Lambda_h+\lambda \Ical_\Hcal)^{-1}}}_{\text{(iii)}}
\end{align}
where the last inequality comes from Cauchy-Schwarz inequality. In the sequel, we aim to bound (iii). We define $\Fcal_h^i = \sigma\big(\{(s_t^i,a_t^i\}_{t=1}^h\big)$ to be the filtration generated by state-action pair before step $h+1$. With Lemma \ref{lem:v-self-concen-rkhs}, we have $$
\text{(iii)}^2\leq  \Ocal\big(H^2\cdot G(n,\underline{\lambda}) + 2H^2 \cdot \log(H/\delta)\big)
$$
with probbaility at least $1-\delta/2$. Therefore we have $$
\big|\widetilde{\P}_h\widetilde{V}_{h+1}(s,a) - \P_h \widetilde{V}_{h+1}(s,a)\big| \leq\|\phi(s,a)\|_{(\Lambda_h + \lambda\Ical_\Hcal)}\cdot\Ocal\big( R_r^2 H^2 \cdot \lambda + H^2\cdot G(n,\lambda) +2H^2\cdot\log(H/\delta) \big)^{1/2} ,
$$
and combined with \eqref{eq:bound-r-rkhs}, we have \begin{align}\label{eq:need-to-bound}
    \big|\big(\widehat{r}_h(s,a) + \widetilde{\P}_h \widetilde{V}_h(s,a)\big) - \big(r_h(s,a) +\P_h\widetilde{V}_h(s,a)\big)\big|&\leq |\widehat{r}_h(s,a)-r_h(s,a)| + \big|\widetilde{\P}_h \widetilde{V}_h(s,a) - \P_h\widetilde{V}_h(s,a)\big|\nonumber\\
    &\leq \|\phi(s,a)\|_{(\Lambda_h + \lambda \Ical_\Hcal)^{-1}}\nonumber\\
    &\quad \cdot \Ocal\big( \lambda R_Q^2 H^2  + H^2 G(n,1+1/n)\nonumber  \\
    &\qquad\quad+\dimsamp^2 H^2e^{2H}|\Acal|^2\log(H\cdot N(\Qcal,\|\cdot\|_\infty, 1/n)/\delta)\big)^{1/2}
\end{align}
with probability at least $1-\delta$ for all $h\in[H]$.  For the constant term on the right-hand side of \eqref{eq:need-to-bound}, we have the following guarantee:
\begin{lemma}\label{lem:upper-bound-beta}
    We have $$
    \lambda R_r^2 H^2  + H^2 G(n,1+1/n) + \dimsamp^2 H^2e^{2H}|\Acal|^2\cdot \log(H\cdot N(\Qcal,\|\cdot\|_\infty, 1/n)/\delta) \leq \beta^2
    $$
    for the three eigenvalue decay conditions discussed in Assumption \ref{ass:eig-decay} and $\beta$ set in Theorem \ref{thm:pessim-vi-rkhs}.
\end{lemma}
\begin{proof}
    See Appendix \ref{sec:prove-beta-rkhs} for details.
\end{proof}
With Lemma \ref{lem:upper-bound-beta}, we prove that $\beta\cdot\|\phi(s,a)\|_{(\Lambda_h+\lambda\Ical_\Hcal)^{-1}}$ is an uncertainty quantifier satisfying condition \ref{eq:uncertainty-quant}. Now we transform it into the desired form in \eqref{eq:penalty-rkhs}. Note that \begin{align}\label{eq:penalty-rkhs-transform-1}
    \|\phi(z)\|^2_{\Lambda_h+ \lambda\Ical_\Hcal} &= \phi(z)^\top\big(\Phi_h^\top\Phi_h + \lambda \Ical_\Hcal )^{-1}\phi(z)\nonumber\\
    &= \frac{1}{\lambda}\big[\phi(z)^\top\phi(z) - \phi(z)^\top\Phi_h^\top\Phi_h(\Phi_h^\top\Phi_h+\lambda\Ical_\Hcal)^{-1}\phi(z)\big]\nonumber\\
    &= \frac{1}{\lambda}\big[K(z,z) - \phi(z)^\top\Phi_h(z)^\top\cdot\Phi_h(\Phi_h^\top\Phi_h+\lambda\Ical_\Hcal)^{-1}\phi(z)\big]\nonumber\\
    &= \frac{1}{\lambda}[K(z,z)- k_h(z)^\top(K_h+\lambda I)^{-1} k_h(z)],
\end{align}
we conclude that \begin{equation}\label{eq:penalty-rkhs-transform-2}
\Gamma_h(z) = \beta\cdot \lambda^{-1/2}\cdot (K(z,z) - k_h(z)^\top(K_h + \lambda I)^{-1}k_h(z))^{1/2},
\end{equation}
and thus we complete the first step.
\paragraph{Step (2).}
The second step is to prove that with $\Gamma_h$ given by \ref{eq:penalty-rkhs} and $\beta$ given by Theorem \ref{thm:pessim-vi-rkhs}, we can give an upper bound for the suboptimality gap. Recall that for $z\in\Zcal$, we define $
\Lambda_h(z) = \Lambda_h + \phi(z)\phi(z)^\top
$, therefore we have $$
\Lambda_h(z) +\lambda \Ical_\Hcal = (\Lambda_h+\lambda \Ical_\Hcal)^{1/2}\big(\Ical_\Hcal + (\Lambda_h+\lambda \Ical_\Hcal)^{-1/2}\phi(z)\phi(z)^\top(\Lambda_h+\lambda \Ical_\Hcal)^{-1/2}\big)(\Lambda_h+\lambda \Ical_\Hcal)^{1/2},
$$
which indicates \begin{align*}
\log\operatorname{det}((\Lambda_h(z)+\lambda \Ical_\Hcal)) &= \log\operatorname{det}((\Lambda_h+\lambda \Ical_\Hcal)) + \log\operatorname{det}\big(\Ical_\Hcal + (\Lambda_h+\lambda \Ical_\Hcal)^{-1/2}\phi(z)\phi(z)^\top(\Lambda_h+\lambda \Ical_\Hcal)^{-1/2}\big)\\
&= \log\operatorname{det}((\Lambda_h+\lambda \Ical_\Hcal)) + \log\big(1+\phi(z)^\top(\Lambda_h+\lambda \Ical_\Hcal)^{-1}\phi(z)\big).
\end{align*}
Since $\phi(z)^\top(\Lambda_h+\lambda \Ical_\Hcal)^{-1}\phi(z)\leq 1$ for $\lambda>1$, we have \begin{align}
\phi(z)^\top\big(\Lambda_h+\lambda\Ical_\Hcal  \big)^{-1}\phi(z) &\leq 2\log\big(1+\phi(z)^\top(\Lambda_h+\lambda \Ical_\Hcal)^{-1}\phi(z)\big)\nonumber\\
    & = 2 \log\operatorname{det}(\Lambda_h(z)+\lambda \Ical_\Hcal) - 2\log\operatorname{det}(\Lambda_h+\lambda \Ical_\Hcal)\nonumber\\
    & = 2\log\operatorname{det}(I + K_h(z)/\lambda) - 2\log(I + K_h/\lambda),
\end{align}
recall that $\Gamma_h(s,a) = \beta \cdot \|\phi(s,a)\|_{(\Lambda_h + \lambda\Ical_\Hcal)^{-1}}$ by \eqref{eq:penalty-rkhs-transform-1} and \eqref{eq:penalty-rkhs-transform-2}, we have \begin{equation}\label{eq:penality-max-info}
    \Gamma_h(s,a)\leq \sqrt{2}\beta\cdot(\log(I + K_h(z)/\lambda) - \log(I + K_h/\lambda))^{1/2} ,
\end{equation}
for all $(s,a)\in\Scal\times\Acal$, and by Theorem \ref{thm:jin-main}, we have \begin{align*}
    \operatorname{SubOpt}(\{\tilde{\pi}_h\})&\leq \sum_{h=1}^H\E_{\pi^*}[\Gamma_h(s_h,a_h)]\\
    &\leq \sum_{h=1}^H\sqrt{2}\beta\cdot\E_{\pi^*}\big[\{\log\operatorname{det}(I + K_h(z_h)/\lambda) - \log\operatorname{det}(I + K_h/\lambda)\}^{1/2}\big]\\
    &\leq \sum_{h=1}^H\sqrt{2}\beta\cdot\big\{\E_{\pi^*}\big[\log\operatorname{det}(I + K_h(z_h)/\lambda) - \log\operatorname{det}(I + K_h/\lambda)\big]\big\}^{1/2}\\
    &= \sum_{h=1}^H \sqrt{2}\beta\cdot \big\{\E_{\pi^*}\big[\phi(\thissa)^\top(\Phi_h^\top\Phi_h+\lambda\Ical_\Hcal)^{-1}\phi(\thissa)\big]\big\}^{1/2}\\
    &= \sum_{h=1}^H \sqrt{2}\beta\operatorname{Tr}\big((K_h+\lambda\Ical_\Hcal)^{-1}\Sigma_h^*\big)^{1/2},
\end{align*}
where the first inequality comes from Theorem \ref{thm:jin-main}, the second from \eqref{eq:penality-max-info}, the third inequality from the feature map representation in \eqref{eq:penalty-rkhs-transform-1}. By Lemma D.19 in \cite{jin2021pessimism}, with $\lambda$ specified in Theorem \ref{thm:pessim-vi-rkhs}, we have $$\sum_{h=1}^H\operatorname{Tr}\big((K_h+\lambda\Ical_\Hcal)^{-1}\Sigma_h^*\big)^{1/2}\leq 4\dimpop $$ for eigenvalue decaying conditions defined in Assumption \ref{ass:eig-decay}. Therefore, for any $\delta\in(0,1)$, we set $\beta$ and $\lambda$ as in Theorem \ref{thm:pessim-vi-rkhs}, then we can guarantee that $$
\operatorname{SubOpt}(\{\tilde{\pi}_h\}_{h\in[H]}) \leq \Ocal\big(\beta\cdot \dimpop\big).
$$
Recall that we define $$
    \beta= \begin{cases}C'' \cdot H\cdot\bigg\{\sqrt{\lambda} R_r +\dimsamp e^H |\Acal|\cdot\log (n R_rH/ \delta)^{1 / 2+1 /(2 \mu)}\bigg\} & \mu \text {-finite spectrum, } \\ C'' \cdot H\cdot\left\{\sqrt{\lambda} R_r +\dimsamp e^H |\Acal|\cdot\log (n R_rH/ \delta)^{1 / 2+1 /(2 \mu)}\right\} & \mu \text {-exponential decay, } \\ C'' \cdot H\cdot\left\{\sqrt{\lambda}R_r + \dimsamp e^H |\Acal|\cdot(nR_r)^{\kappa^*} \cdot \sqrt{\log (nR_rH/\delta)}\right\} & \mu \text {-polynomial decay,}\end{cases}
    $$
    we therefore conclude the proof of Theorem \ref{thm:pessim-vi-rkhs}.
\subsection{Proof for Lemma \ref{lem:upper-bound-beta}}\label{sec:prove-beta-rkhs}
We prove Lemma \ref{lem:upper-bound-beta} by discussing the eigenvalue decaying conditions in Assumption \ref{ass:eig-decay} respectively. 
\paragraph{(i):$\mu$-finite spectrum.}
In this case, since $1+1 / n \in[1,2]$, by Lemma \ref{lem:gn-rkhs}, there exists some absolute constant $C$ that only depends on $d, \mu$ such that
$$
G(n, 1+1 / n) \leq C \cdot \mu \cdot \log n,
$$
and by Lemma \ref{lem:cover-rkhs}, there exists an absolute constant $C'$ such that $$
\log N(\Qcal,\|\cdot\|_\infty ,1/n) \leq C' \cdot \mu \cdot\left[\log (nR_rH)+C_4\right],
$$
Hence we could set $\beta=c \cdot H \cdot \big(\sqrt{\lambda}R_r + \dimsamp e^H |\Acal|\cdot\sqrt{\mu \log(nR_rH/\delta)}\big)$ for some sufficiently large constant $c>0$.
\paragraph{(ii): $\mu$-exponential decay.} By Lemma \ref{lem:gn-rkhs}, there exists some absolute constant $C$ that only depends on $d, \gamma$ such that
$$
G(n, 1+1 / n) \leq C \cdot(\log n)^{1+1 / \mu},
$$
and by Lemma \ref{lem:cover-rkhs}, there exists an absolute constant $C'$ such that
$$
\log N(\Qcal,\|\cdot\|_\infty ,1/n) \leq C'\cdot\log(nR_r)^{1+1/\mu},
$$
We can thus choose $\beta=c \cdot H\cdot(\sqrt{\lambda} R_r +\dimsamp e^H |\Acal|\cdot\log (n R_rH/ \delta)^{1 / 2+1 /(2 \mu)})$ for some sufficiently large absolute constant $c>0$ depending on $d, \mu, C_1, C_2$ and $C_\psi$.
\paragraph{(iii): $\mu$-polynomial decay.} By Lemma \ref{lem:gn-rkhs}, there exists some absolute constant $C$ that only depends on $d, \mu$ such that
$$
G(n, 1+1 / n) \leq C \cdot n^{\frac{d+1}{\mu+d}} \cdot \log n,
$$
and by Lemma \ref{lem:cover-rkhs}, there exists an absolute constant $C'$ such that
$$
\log N(\Qcal,\|\cdot\|_\infty ,1/n) \leq C'\cdot(nR_r)^{2 /[\mu \cdot(1-2 \tau)-1]} \cdot \log (nR_r),
$$
Thus, it suffices to choose $\beta=c \cdot H\cdot (\sqrt{\lambda}R_r + \dimsamp e^H |\Acal|\cdot(nR_r)^{\kappa^*} \cdot \sqrt{\log (nR_rH/\delta)})$, where $c>0$ is a sufficiently large absolute constant depending on $d, \mu$. Here $$
\kappa^*=\frac{d+1}{2(\mu+d)}+\frac{1}{\mu(1-2 \tau)-1}.
$$

\section{Auxiliary Lemma}
The following lemma is useful in the proof of Lemma \ref{lem:cauchy-matrix}.
\begin{lemma}\label{lem:semi-defi-positive}
For three symmetrical matrices $A, B$ and $C$, suppose $A\succeq B$ and $C\succeq 0$, we have $$
\langle A, C \rangle \geq \langle B, C \rangle.
$$

\end{lemma}
\begin{proof}
Consider $$
\langle A-B, C \rangle = \operatorname{tr}
\big((A-B)C\big).
$$
Note that since $C$ is positive definite, we have a real symmetrical matrix $H$ such that $C= H^2$. Therefore we have $$
\operatorname{tr}
\big((A-B)C\big) = \operatorname{tr}
\big(H(A-B)H\big).
$$
Denote $H$ by $(h_1,\cdots, h_d)$, we then have $$
\operatorname{tr}
\big(H(A-B)H\big) = \sum_{i=1}^d h_i^\top (A-B) h_i,
$$
and by $A-B$ being semi-definite positive we conclude the proof.
\end{proof}
The following lemma is useful when upper bounding the self-normalizing sequence.
\begin{lemma}\label{lem:cauchy-matrix}
    For real numbers $x_1,x_2,...,x_n$ and real vectors $\bm{c}_1, \bm{c}_2,...,\bm{c}_n \in \Hcal$, where $\Hcal$ is a Hilbert space. If $\sum_{i=1}^n x_i^2 \leq C$, where $C>0$ is a positive constant, then $$
    (\sum_{i=1}^n x_i \bm{c}_i)(\sum_{i=1}^n x_i \bm{c}_i)^\top 	\preceq C\cdot \sum_{i=1}^n \bm{c_ic_i^\top}.
    $$
\end{lemma}
\begin{proof}
    Consider an arbitrary vector $\bm{y}\in \Hcal$. We have \begin{align*}
            y^\top(\sum_{i=1}^n x_i \bm{c}_i)(\sum_{i=1}^n x_i \bm{c}_i)^\top y &=\bigg\|\sum_{i=1}^n x_i \cdot(\bm{c}_i\cdot \bm{y})\bigg\|_{\Hcal}^2 \\
            &\leq \bigg(\sum_{i=1}^n x_i^2\bigg) \bigg(\sum_{i=1}^n (\bm{c}_i\cdot \bm{y})^2\bigg)\\
            &\leq C\cdot \bigg(\bm{y}^\top\sum_{i=1}^n \bm{c_i}\bm{c_i}^\top\bm{y}\bigg),
    \end{align*}
    since this holds for all $\bm{y} \in \Hcal$ we conclude the proof.

\end{proof}
The following lemma upperly bounds the bracketling number of a parametrized function class by the covering number of the paramter class when it is Lipschitz-continuous to the parameter.
\begin{lemma}\label{lem:cover-bracket}
Consider a class $\Fcal$ of functions ${m_\theta : \theta\in \Theta}$ indexed by a parameter $\theta$ in an arbitrary
index set $\Theta$ with a metric $d$. Suppose that the dependence on $\theta$ is Lipschitz in the sense that $$\left|m_{\theta_1}(x)-m_{\theta_2}(x)\right| \leq d\left(\theta_1, \theta_2\right) F(x)$$ for some function $F: \Xcal\rightarrow\R$, for every $\theta_1,\theta_2\in\Theta$ and $x\in\Xcal$. Then, for any norm $\|\cdot\|$, the bracketing
numbers of this class are bounded by the covering numbers:$$
N_{[]}(\Fcal,\|\cdot\|, 2\epsilon\|F\|) \leq N(\Theta,d,\epsilon).
$$
\end{lemma}
\begin{proof}
See Lemma 2.14 in \cite{sen2018gentle} for details.
\end{proof}
The following two lemmas, obtained from \cite{abbasi2011improved}, establishes the concentration of self-normalized processes.
\begin{lemma}\label{lem:self-norm-concen}[Concentration of Self-Normalized Processes, \citep{abbasi2011improved}] Let $\left\{\epsilon_t\right\}_{t=1}^{\infty}$ be a real-valued stochastic process that is adaptive to a filtration $\left\{\mathcal{F}_t\right\}_{t=0}^{\infty}$. That is, $\epsilon_t$ is $\mathcal{F}_t$-measurable for all $t \geq 1$. Moreover, we assume that, for any $t \geq 1$, conditioning on $\mathcal{F}_{t-1}, \epsilon_t$ is a zero-mean and $\sigma$-subGaussian random variable such that
$$
\mathbb{E}\left[\epsilon_t \mid \mathcal{F}_{t-1}\right]=0 \quad \text { and } \quad \mathbb{E}\left[\exp \left(\lambda \epsilon_t\right) \mid \mathcal{F}_{t-1}\right] \leq \exp \left(\lambda^2 \sigma^2 / 2\right), \quad \forall \lambda \in \mathbb{R} .
$$
Besides, let $\left\{\phi_t\right\}_{t=1}^{\infty}$ be an $\mathbb{R}^d$-valued stochastic process such that $\phi_t$ is $\mathcal{F}_{t-1}$-measurable for all $t \geq 1$. Let $M_0 \in \mathbb{R}^{d \times d}$ be a deterministic and positive-definite matrix, and we define $M_t=M_0+\sum_{s=1}^t \phi_s \phi_s^{\top}$ for all $t \geq 1$. Then for any $\delta>0$, with probability at least $1-\delta$, we have for all $t \geq 1$ that
$$
\left\|\sum_{s=1}^t \phi_s \cdot \epsilon_s\right\|_{M_t^{-1}}^2 \leq 2 \sigma^2 \cdot \log \left(\frac{\operatorname{det}\left(M_t\right)^{1 / 2} \operatorname{det}\left(M_0\right)^{-1 / 2}}{\delta}\right) .
$$
\begin{proof}
    See Theorem 1 of \cite{abbasi2011improved} for detailed proof.
\end{proof} 
    
\end{lemma}
\begin{lemma}[Concentration of Self-Normalized Process for RKHS, \citep{chowdhury2017kernelized}]\label{lem:self-norm-concen-rkhs}
 Let $\mathcal{H}$ be an RKHS defined over $\mathcal{X} \subseteq \mathbb{R}^d$ with kernel function $K(\cdot, \cdot): \mathcal{X} \times \mathcal{X} \rightarrow \mathbb{R}$. Let $\left\{x_\tau\right\}_{\tau=1}^{\infty} \subset \mathcal{X}$ be a discrete-time stochastic process that is adapted to the filtration $\left\{\mathcal{F}_t\right\}_{t=0}^{\infty}$. Let $\left\{\epsilon_\tau\right\}_{\tau=1}^{\infty}$ be a real-valued stochastic process such that (i) $\epsilon_\tau \in \mathcal{F}_\tau$ and (ii) $\epsilon_\tau$ is zero-mean and $\sigma$-sub-Gaussian conditioning on $\mathcal{F}_{\tau-1}$, i.e.,
$$
\mathbb{E}\left[\epsilon_\tau \mid \mathcal{F}_{\tau-1}\right]=0, \quad \mathbb{E}\left[e^{\lambda \epsilon_\tau} \mid \mathcal{F}_{\tau-1}\right] \leq e^{\lambda^2 \sigma^2 / 2}, \quad \forall \lambda \in \mathbb{R}
$$
Moreover, for any $t \geq 2$, let $E_t=\left(\epsilon_1, \ldots, \epsilon_{t-1}\right)^{\top} \in \mathbb{R}^{t-1}$ and $K_t \in \mathbb{R}^{(t-1) \times(t-1)}$ be the Gram matrix of $\left\{x_\tau\right\}_{\tau \in[t-1]}$. Then for any $\eta>0$ and any $\delta \in(0,1)$, with probability at least $1-\delta$, it holds simultaneously for all $t \geq 1$ that
$$
E_t\left[\left(K_t+\eta \cdot I\right)^{-1}+I\right]^{-1} E_t \leq \sigma^2 \cdot \log \operatorname{det}\left[(1+\eta) \cdot I+K_t\right]+2 \sigma^2 \cdot \log (1 / \delta)
$$
\end{lemma}
\begin{proof}
    See Theorem 1 in \cite{chowdhury2017kernelized} for detailed proof.
\end{proof}
The following theorem gives a uniform bound for a set of self-normalizing sequences, whose proof can be found in Appendix B.2, \cite{jin2021pessimism}. It is useful for uniformly bounding the self-normalizing sequence in pessimistic value iteration, both for linear model MDP class:
\begin{theorem}\label{thm:ullm-pessi}
    For $h\in[H]$, we define the function class $\mathcal{V}_h(R, B, \lambda)=\left\{V_h(x ; \theta, \beta, \Sigma): \mathcal{S} \rightarrow[0, H]\right.$ with $\left.\|\theta\| \leq R, \beta \in[0, B], \Sigma \succeq \lambda \cdot I\right\},$
where $$V_h(x ; \theta, \beta, \Sigma)=\max _{a \in \mathcal{A}}\left\{\min \left\{\phi(x, a)^{\top} \theta-\beta \cdot \sqrt{\phi(x, a)^{\top} \Sigma^{-1} \phi(x, a)}, H-h+1\right\}_{+}\right\},$$
then we have \begin{align*}
&\sup _{V \in \mathcal{V}_{h+1}(R, B, \lambda)}\left\|\sum_{i=1}^n \phi\left(s_h^i,a_h^i\right) \cdot \big(V(s_{h+1}^i) - \E[V(s_{h+1})\mid s_h^i,a_h^i]\big)\right\|^2_{(\Lambda_h+\lambda I)^{-1}}\\
&\qquad\leq 8\epsilon^2 n^2/\lambda + 2H^2 \cdot \big(2\cdot \log(\Ncal/\delta)+d\cdot \log(1+n/\lambda)\big),
\end{align*}
holds with probability at least $1-\delta$ for every $\epsilon>0$.
Here $$
\log(\Ncal)\leq d \cdot \log (1+4 R / \varepsilon)+d^2 \cdot \log (1+8 d^{1 / 2} B^2 /(\varepsilon^2 \lambda)).
$$
\end{theorem}
\begin{proof}
    See Appendix B.2 in \cite{jin2021pessimism} for details.
\end{proof}

\begin{lemma}[Lemma D.5 in \cite{yang2020provably}]\label{lem:gn-rkhs}
     Let $\mathcal{Z}$ be a compact subset of $\mathbb{R}^d$ and $K: \mathcal{Z} \times$ $\mathcal{Z} \rightarrow \mathbb{R}$ be the RKHS kernel of $\mathcal{H}$. We assume that $K$ is a bounded kernel in the sense that $\sup _{z \in \mathcal{Z}} K(z, z) \leq 1$, and $K$ is continuously differentiable on $\mathcal{Z} \times \mathcal{Z}$. Moreover, let $T_K$ be the integral operator induced by $K$ and the Lebesgue measure on $\mathcal{Z}$, whose definition is given in \eqref{eq:mercer-integral}. Let $\left\{\sigma_j\right\}_{j \geq 1}$ be the eigenvalues of $T_K$ in the descending order. We assume that $\left\{\sigma_j\right\}_{j \geq 1}$ satisfy either one of the following three eigenvalue decay conditions:
(i) $\mu$-finite spectrum: We have $\sigma_j=0$ for all $j \geq \mu+1$, where $\mu$ is a positive integer.
(ii) $\mu$-exponential eigenvalue decay: There exist constants $C_1, C_2>0$ such that $\sigma_j \leq C_1 \exp \left(-C_2\right.$. $j^\mu$ ) for all $j \geq 1$, where $\mu>0$ is positive constant.
(iii) $\mu$-polynomial eigenvalue decay: There exists a constant $C_1$ such that $\sigma_j \geq C_1 \cdot j^{-\mu}$ for all $j \geq 1$, where $\mu \geq 2+1 / d$ is a constant.

Let $\sigma$ be bounded in interval $\left[c_1, c_2\right]$ with $c_1$ and $c_2$ being absolute constants. Then, for conditions (i)-(iii) respectively, we have
$$
G(n,\lambda) \leq \begin{cases}C_n \cdot \mu \cdot \log n & \mu \text {-finite spectrum, } \\ C_n \cdot(\log n)^{1+1 / \mu} & \mu \text {-exponential decay, } \\ C_n \cdot n^{(d+1) /(\mu+d)} \cdot \log n & \mu \text {-polynomial decay, }\end{cases}
$$
where $C_n$ is an absolute constant that depends on $d, \mu, C_1, C_2, C, c_1$, and $c_2$.
\end{lemma}
\begin{proof}
    See Lemma D.5 in \cite{yang2020provably} for details.
\end{proof}
\begin{lemma}[ $\ell_{\infty}$-norm covering number of RKHS ball]\label{lem:cover-rkhs}
 For any $\epsilon \in(0,1)$, we let $N(\Qcal,\|\cdot\|_\infty ,\epsilon)$ denote the $\epsilon$-covering number of the RKHS norm ball $\Qcal = \left\{f \in \mathcal{H}:\|f\|_{\mathcal{H}} \leq R\right\}$ with respect to the $\ell_{\infty}$-norm. Consider the three eigenvalue decay conditions given in Assumption \ref{ass:eig-decay}. Then, under Assumption \ref{ass:eig-decay}, there exist absolute constants $C_3$ and $C_4$ such that
$$
\log N(\Qcal,\|\cdot\|_\infty ,\epsilon) \leq \begin{cases}C_3 \cdot \mu \cdot\left[\log (R / \epsilon)+C_4\right] & \mu \text {-finite spectrum, } \\ C_3 \cdot\left[\log (R / \epsilon)+C_4\right]^{1+1 / \mu} & \mu \text {-exponential decay, } \\ C_3 \cdot(R / \epsilon)^{2 /[\mu \cdot(1-2 \tau)-1]} \cdot\left[\log (R / \epsilon)+C_4\right] & \mu \text {-polynomial decay, }\end{cases}
$$
where $C_3$ and $C_4$ are independent of $n, H, R$, and $\epsilon$, and only depend on absolute constants $C_\psi$, $C_1, C_2, \mu$, and $\tau$ specified in Assumption \ref{ass:eig-decay}.
\end{lemma} 
\begin{proof}
    See Lemma D.2 in \cite{yang2020provably} for details.
\end{proof}
\end{document}